\documentclass{article}

\usepackage{microtype}
\usepackage{graphicx}
\usepackage{bm}
\usepackage{algorithm}
\usepackage{algorithmic}
\usepackage{multirow}
\usepackage{amsthm}
\usepackage{amsmath}
\usepackage{subfigure}
\usepackage{amssymb}
\usepackage{booktabs} 

\usepackage{balance}
\usepackage{amsmath}
\usepackage{float}
\usepackage{subeqnarray}
\usepackage{cases}
\usepackage{tabularx}

\newtheorem{thm}{Theorem}

\newtheorem{prop}{Proposition}
\newtheorem{assum}{Assumption}
\newtheorem{lemma}{Lemma}
\newtheorem{defi}{Definition}
\newtheorem{remark}{Remark}

\newtheorem{example}{Example}

\def\x{\mathbf{x}}
\def\y{\mathbf{y}}
\def\u{\mathbf{u}}
\def\v{\mathbf{v}}
\def\X{\mathcal{X}}
\def\S{\mathcal{S}}
\def\T{\mathcal{T}}

\newcommand{\tabincell}[2]{\begin{tabular}{@{}#1@{}}#2\end{tabular}}

\usepackage[accepted]{icml2020}

\icmltitlerunning{A Generic First-Order Algorithmic Framework for Bi-Level Programming Beyond Lower-Level Singleton} 

\begin{document}
	
\twocolumn[
\icmltitle{A Generic First-Order Algorithmic Framework for Bi-Level Programming \\Beyond Lower-Level Singleton}




\icmlsetsymbol{equal}{*}

\begin{icmlauthorlist}
	\icmlauthor{Risheng Liu}{1,2}
	\icmlauthor{Pan Mu}{1,2}
	\icmlauthor{Xiaoming Yuan}{4}
	\icmlauthor{Shangzhi Zeng}{4}
	\icmlauthor{Jin Zhang}{5}
\end{icmlauthorlist}

\icmlaffiliation{1}{DUT-RU International School of Information Science and Engineering, Dalian University of Technology.}
\icmlaffiliation{2}{Key Laboratory for Ubiquitous Network and Service Software of Liaoning Province.}
\icmlaffiliation{4}{Department of Mathematics, The University of Hong Kong.}
\icmlaffiliation{5}{SUSTech International Center for Mathematics and Department of Mathematics, Southern University of Science and Technology}
\icmlcorrespondingauthor{Jin Zhang}{zhangj9@sustech.edu.cn}

\vskip 0.3in
]



\printAffiliationsAndNotice{}  

\begin{abstract}
In recent years, a variety of gradient-based bi-level optimization methods have been developed for learning tasks. However, theoretical guarantees of these existing approaches often heavily rely on the simplification that for each fixed upper-level variable, the lower-level solution must be a singleton (a.k.a., Lower-Level Singleton, LLS). In this work, by formulating bi-level models from the optimistic viewpoint and aggregating hierarchical objective information, we establish Bi-level Descent Aggregation (BDA), a flexible and modularized algorithmic framework for bi-level programming. Theoretically, we derive a new methodology to prove the convergence of BDA without the LLS condition. Furthermore, we improve the convergence properties of conventional first-order bi-level schemes (under the LLS simplification) based on our proof recipe. Extensive experiments justify our theoretical results and demonstrate the superiority of the proposed BDA for different tasks, including hyper-parameter optimization and meta learning.
\end{abstract}

\section{Introduction}

Bi-Level Programs (BLPs) are mathematical programs with optimization problems in their constraints and recently have been recognized as powerful theoretical tools for a variety of machine learning applications. Mathematically, BLPs can be (re)formulated as the following optimization problem:
\begin{equation}
\min\limits_{\mathbf{x}\in\mathcal{X},\mathbf{y}\in\mathbb{R}^m}F(\x,\y), \ s.t. \ \mathbf{y}\in\mathcal{S}(\mathbf{x}),\label{eq:blp}
\end{equation}
where the Upper-Level (UL) objective $F$ is a jointly continuous function, the UL constraint $\mathcal{X}$ is a compact set, and the set-valued mapping $\mathcal{S}(\mathbf{x})$ indicates the parameterized solution set of the Lower-Level (LL) subproblem. Without loss of generality, we consider the following LL subproblem:
\begin{equation}
\S(\x)=\arg\min\limits_{\y} f(\x,\y),\label{eq:lower-level}
\end{equation}
where $f$ is another jointly continuous function. Indeed, the BLPs model formulated in Eqs.~\eqref{eq:blp}-\eqref{eq:lower-level} is a hierarchical optimization problem with two coupled variables $(\x,\y)\in\mathbb{R}^n\times\mathbb{R}^{m}$. Specifically, given the UL variable $\x$ from the feasible set $\X$ (i.e., $\x\in\X$), the LL variable $\y$ is an optimal solution of the LL subproblem governed by $\x$ (i.e., $\y\in\S(\x)$). 
Due to the hierarchical structure and the complicated dependency between UL and LL variables, solving the above BLPs problem is challenging in general, especially when the LL solution set $\S(\x)$ in Eq.~\eqref{eq:lower-level} is not a singleton~\cite{jeroslow1985polynomial,dempe2018bilevel}. In this work, we always call the condition that $\S(\x)$ is a singleton as Lower-Level Singleton (or LLS for short).

\begin{table*}[htb]
	\small
	\caption{Comparing the convergence results (together with properties required by the UL and LL subproblems) between BDA and the existing bi-level FOMs in different scenarios (i.e., BLPs with and without LLS condition). Here $\xrightarrow[]{s}$ and $\xrightarrow[]{u}$ represent the subsequential and uniform convergence, respectively. The superscript $^*$ denotes that it is the true optimal variables/values.
	}\label{tab:result}
	\centering
	\vskip 0.12in
	\renewcommand\arraystretch{1.1} 
	\begin{tabular}{| c | c | c | c |}
		\hline
		Alg. &   & {LLS} &  {w/o LLS} \\
		\hline
		\multirow{3}{*}{\tabincell{c}{Existing\\ bi-level FOMs}} & {UL} &   $F(\x,\cdot)$ is Lipschitz continuous. & \\
		\cline{2-3}
		& \multirow{1}{*}{{LL}} & $\{\y_K(\x)\}$ is uniformly bounded on $\mathcal{X}$, $\y_K(\x)\xrightarrow[]{u}\y^*(\x)$. & Not Available\\
		\cline{2-3}
		& \multicolumn{2}{c|}{Main results:  $\x_{K} \xrightarrow[]{s} \x^*$, $\inf_{\x \in \mathcal{X}}\varphi_K(\x) \to \inf_{\x \in \mathcal{X}}\varphi(\x)$. }& \\
		\hline
		& \multirow{2}{*}{{UL}} & \multirow{2}{*}{ $F(\x,\cdot)$ is Lipschitz continuous.} & 
		$F(\x,\cdot)$ is Lipschitz continuous, \\
		&&&${L}_F$-smooth, and $\sigma$-strongly convex. \\
		\cline{2-4}
		BDA & \multirow{2}{*}{{LL}} & $\{\y_K(\x)\}$ is uniformly bounded on $\mathcal{X}$, $f(\x,\y_K(\x))\xrightarrow[]{u} f^*(\x)$, & 
		$f(\x,\cdot)$ is ${L}_f$-smooth and convex,\\
		&    & $f(\x,\y)$ is level-bounded in $\y$ locally uniformly in $\x\in\X$. &  $\mathcal{S}(\x)$ is continuous.\\
		\cline{2-4}
		& \multicolumn{3}{c|}{Main results:  $\x_{K} \xrightarrow[]{s}\x^*$, $\inf_{\x \in \mathcal{X}}\varphi_K(\x) \to \inf_{\x \in \mathcal{X}}\varphi(\x)$.}   \\
		\hline
	\end{tabular}
\end{table*}

\subsection{Related Work}

Although early works on BLPs can date back to the nineteen seventies~\cite{dempe2018bilevel}, it was not until the last decade that a large amount of bi-level optimization models were established to formulate specific machine learning problems, include meta learning~\cite{franceschi2018bilevel,rajeswaran2019meta,zugner2019adversarial}, hyper-parameter optimization~\cite{franceschi2017forward,okuno2018hyperparameter,mackay2019self}, reinforcement learning~\cite{yang2019provably}, generative adversarial learning~\cite{pfau2016connecting}, and image processing~\cite{kunisch2013bilevel,de2017bilevel}, just to name a few. 

A large number of optimization techniques have been developed to solve BLPs in Eqs.~\eqref{eq:blp}-\eqref{eq:lower-level}. For example, the works in \cite{kunapuli2008classification,moore2010bilevel,okuno2018hyperparameter} aimed to reformulate the original BLPs in Eqs.~\eqref{eq:blp}-\eqref{eq:lower-level} as a 
single-level optimization problem based on the first-order optimality conditions. However, these approaches involve too many auxiliary variables, thus are not applicable for complex machine learning tasks. 

Recently, gradient-based First-Order Methods (FOMs) have also been investigated to solve BLPs. The key idea underlying these approaches is to hierarchically calculate gradients of UL and LL objectives. Specifically, the works in~\cite{maclaurin2015gradient,franceschi2017forward,franceschi2018bilevel} first calculate gradient representations of the LL objective and then perform either reverse or forward gradient computations (a.k.a., automatic differentiation, based on the LL gradients) for the UL subproblem. It is known that the reverse mode is related to the back-propagation through time while the forward mode actually appears to the standard chain rule~\cite{franceschi2017forward}. In fact, similar ideas have also been used in~\cite{jenni2018deep,zugner2019adversarial,rajeswaran2019meta}, but with different specific implementations. 
In~\cite{shaban2018truncated}, a truncated back-propagation scheme is adopted to improve the scale issue for the LL gradient updating. Furthermore, the works in~\cite{lorraine2018stochastic,mackay2019self} trained a so-called hyper-network to map LL gradients for their hierarchical optimization. 

Although widely used in different machine learning applications, theoretical properties of these bi-level FOMs are still not convincing (summarized in Table~\ref{tab:result}). Indeed, all of these methods require the LLS constraint in Eq.~\eqref{eq:lower-level} to simplify their optimization process and theoretical analysis. For example, to satisfy such restrictive condition, existing works~\cite{franceschi2018bilevel,shaban2018truncated} have to enforce a (local) strong convexity assumption to their LL subproblem, which is actually too tough to be satisfied in real-world complex tasks. 

\subsection{Our Contributions}

This work proposes Bi-level Descent Aggregation (BDA), a generic bi-level first-order algorithmic framework that is flexible and modularized to handle BLPs in Eqs.~\eqref{eq:blp}-\eqref{eq:lower-level}. Unlike the above existing bi-level FOMs, which require the LLS assumption on Eq.~\eqref{eq:lower-level} and separate the original model into two single-level subproblems, our BDA investigates BLPs from the optimistic viewpoint and develop a new hierarchical optimization scheme, which consists of a single-level optimization formulation for the UL variable $\x$ and a simple bi-level optimization formulation for the LL variable $\y$. Theoretically, we establish a general proof recipe to analyze the convergence behaviors of these bi-level FOMs. We prove that the convergence of BDA can be strictly guaranteed in the absence of the restrictive LLS condition. Furthermore, we demonstrate that the strong convexity of the LL objective (required in previous theoretical analysis~\cite{franceschi2018bilevel}) is actually non-essential for these existing LLS-based bi-level FOMs, such as \cite{domke2012generic,maclaurin2015gradient,franceschi2017forward,franceschi2018bilevel,shaban2018truncated}. Table~\ref{tab:result} compares the convergence results of BDA and the existing approaches. It can be seen that in LLS scenario, BDA and the existing methods share the same requirements for the UL subproblem. However, for the LL subproblem, assumptions required in previous approaches are essentially more restrictive than that in BDA. More importantly, when solving BLPs without LLS, no theoretical results can be obtained for these classical methods. Fortunately, BDA can still obtain the same convergence properties as that in LLS scenario. The contributions can be  summarized as:  
\begin{itemize}
	\item A counter-example (i.e.,  Example~\ref{CExam}) explicitly indicates the importance of the LLS condition for the existing bi-level FOMs. In particular, we investigate their iteration behaviors and reach the conclusion that using these approaches in the absence of the LLS condition may lead to incorrect solutions.
	\item 
	By formulating BLPs in Eqs.~\eqref{eq:blp}-\eqref{eq:lower-level} from the viewpoint of optimistic bi-level, BDA provides a generic bi-level algorithmic framework. Embedded with a specific gradient-aggregation-based iterative module, BDA is applicable to a variety of learning tasks.
	\item A general proof recipe is established to analyze the convergence behaviors of bi-level FOMs. We strictly prove the convergence of BDA without the LLS assumption. Furthermore, we revisit and improve the convergence properties of the existing bi-level FOMs in the LLS scenario.  
\end{itemize}

\section{First-Order Methods for BLPs}

\subsection{Solution Strategies with Lower-Level Singleton}\label{sec:rhg}
As aforementioned, a number of FOMs have been proposed to solve BLPs in Eqs.~\eqref{eq:blp}-\eqref{eq:lower-level}. However, these existing methods all rely on the uniqueness of $\S(\x)$ (i.e., LLS assumption). That is, rather than considering the original BLPs in Eqs.~\eqref{eq:blp}-\eqref{eq:lower-level}, they actually solve the following simplification:
\begin{equation}
\min\limits_{\mathbf{x}\in\mathcal{X}} F(\mathbf{x},\mathbf{y}), \ s.t. \ \mathbf{y}=\arg\min\limits_{\mathbf{y}} f(\mathbf{x},\mathbf{y}),\label{eq:blp-singleton}
\end{equation}
where the LL subproblem only has one single solution for a given $\x$. By
considering $\y$ as a function of $\x$, the idea behind these approaches is to take a gradient-based first-order scheme (e.g, gradient descent, stochastic gradient descent, or their variations) on the LL subproblem. Therefore, with the initialization point $\y_0$, a sequence $\{\mathbf{y}_{k}\}_{k=0}^K$ parameterized by $\x$ can be generated, e.g.,
\begin{equation}
\mathbf{y}_{k+1}=\mathbf{y}_{k}-s_l\nabla_{\mathbf{y}} f(\mathbf{x},\mathbf{y}_{k}), \ k=0,\cdots,K-1,\label{eq:gradient_f}
\end{equation}
where $s_l>0$ is an appropriately chosen step size. Then by considering $\y_K(\x)$ (i.e., the output of Eq.~\eqref{eq:gradient_f} for a given $\x$) as an approximated optimal solution to the LL subproblem, we can incorporate $\y_K(\x)$ into the UL objective and obtain a single-level approximation model, i.e., $
\min_{\mathbf{x}\in\mathcal{X}} F(\mathbf{x},\y_K(\mathbf{x}))
$.
Finally, by unrolling the iterative update scheme in Eq.~\eqref{eq:gradient_f}, we can calculate the derivative of $F(\mathbf{x},\y_K(\mathbf{x}))$ (w.r.t. $\x$) to optimize Eq.~\eqref{eq:blp-singleton} by automatic differentiation techniques~\cite{franceschi2017forward,baydin2017automatic}. 

\subsection{Fundamental Issues and Counter-Example}\label{sec:ce}

It can be observed that the LLS condition fairly matters for the validation of the existing bi-level FOMs. However, such   singleton assumption on the solution set of the LL subproblem is actually too restrictive to be satisfied, especially in real-world applications. 
In this subsection, we design an interesting counter-example (Example~\ref{CExam} below) to illustrate such invalidation of these conventional gradient-based bi-level schemes in the absence of the LLS condition. 

\begin{example}(Counter-Example)\label{CExam}
With $\x\in[-100,100]$ and $\y\in\mathbb{R}^2$, we consider the following BLPs problem:
\begin{equation}
\begin{array}{c}
\min\limits_{\x\in[-100,100]}\frac{1}{2}(\x-[\y]_2)^2+\frac{1}{2}([\y]_{1}-1)^2,\\
s.t. \ \y\in\arg\min\limits_{\y\in\mathbb{R}^2}\frac{1}{2}[\y]_1^2 - \x[\y]_1,
\end{array}\label{eq:ce}
\end{equation}
where $[\cdot]_i$ denotes the $i$-th element of the vector. By simple calculation, we know that the optimal solution of Eq.~\eqref{eq:ce} is $\x^*=1, \y^* = (1,1)$. However, if adopting the existing gradient-based scheme in Eq.~\eqref{eq:gradient_f} with initialization $\y_0=(0,0)$ and varying step size $s_l^k  \in (0,1)$, we have that $[\y_K]_1 = (1-\prod_{k=0}^{K-1}(1-s_l^k))\x$ and $[\y_K]_2=0$. Then the approximated problem of Eq.~\eqref{eq:ce} amounts to $$\min_{\mathbf{x}\in [-100,100]} F(\x,\y_K)=\frac{1}{2}\x^2+\frac{1}{2}( (1-\prod_{k=0}^{K-1}(1-s_l^k)) \x- 1 )^2.$$ By defining $\varphi_K(\x) = F(\x,\y_K)$, we have
$$
\x_K^{*}=\arg\min_{\x\in[-100,100]}\phi_K(\x)=\frac{(1- \prod_{k=0}^{K-1}(1-s_l^k) )}{1+(1- \prod_{k=0}^{K-1}(1-s_l^k) )^2}.
$$
It is easy to check that 
$$
0 \le \liminf_{K \rightarrow \infty} \prod_{k=0}^{K-1}(1-s_l^k) \le \limsup_{K \rightarrow \infty} \prod_{k=0}^{K-1}(1-s_l^k)\le 1, 
$$ 
then we have 
$\limsup_{K \rightarrow \infty} \frac{(1- \prod_{k=0}^{K-1}(1-s_l^k) )}{1+(1- \prod_{k=0}^{K-1}(1-s_l^k) )^2}  \le \frac{1}{2}.$ 
So $\x_K^*$  cannot converge to the true solution (i.e., $\x^* = 1$).
\end{example}

\begin{remark}
	The UL objective $F$ is indeed a function of both the UL variable $\x$ and the LL variable $\y$. Conventional bi-level FOMs only use the gradient information of the LL subproblem to update $\mathbf{y}$. Thanks to the LLS assumption, for fixed UL variable $\x$, the LL solution $\y$ can be uniquely determined. Thus the sequence $\{\mathbf{y}_{k}\}_{k=0}^K$ could converge to the true optimal solution, that minimizes both the LL and UL objectives. However, when LLS is absent, $\{\mathbf{y}_{k}\}_{k=0}^K$ may easily fail to converge to the true solution. Therefore, $\x_K^*$ may tend to be incorrect limiting points. Fortunately, we will demonstrate in Sections~\ref{sec:bda} and ~\ref{sec:exp} that the example in Eq.~\eqref{eq:ce} can be efficiently solved by our proposed BDA. 
\end{remark}

\section{Bi-level Descent Aggregation}\label{sec:bda}

In contrast to previous works, which only consider simplified BLPs with the LLS assumption in Eq.~\eqref{eq:blp-singleton}, we propose a new algorithmic framework, named Bi-level Descent Aggregation (BDA), to handle more generic BLPs in Eqs.~\eqref{eq:blp}-\eqref{eq:lower-level}.

\subsection{Optimistic Bi-level Algorithmic Framework}

In fact, the situation becomes intricate if the LL subproblem is not uniquely solvable for each $\x\in\X$. In this work, we consider BLPs from the optimistic bi-level viewpoint\footnote{For more theoretical details of optimistic BLPs, we refer to~\cite{dempe2018bilevel} and the references therein.}, thus for any given $\x$, we expect to choose the LL solution $\y\in\S(\x)$ that can also lead to the best objective function value for the UL objective (i.e., $F(\x,\cdot)$). Inspired by this observation, we can reformulate Eqs.~\eqref{eq:blp}-\eqref{eq:lower-level} as 
\begin{equation}
\min\limits_{\x\in\X}\varphi(\x), \
\mbox{with} \ \varphi(\x) = \inf\limits_{\y\in\S(\x)}F(\x,\y).\label{eq:oblp}
\end{equation}
Such reformulation reduces BLPs to a single-level problem $\min_{\mathbf{x}\in\mathcal{X}}\varphi(\x)$ w.r.t. the UL variable $\x$. While for any given $\x$, $\varphi$ actually turns out to be the value function of a simple bi-level problem w.r.t. the LL variable $\y$, i.e.,
\begin{equation}
\min_{\y}F(\x,\y), \ s.t. \ \y\in\S(\x), \ \mbox{(with fixed $\x$)}.\label{eq:simple-blp}
\end{equation}
Based on the above analysis, we actually could update $\y$ by
\begin{equation}
\mathbf{y}_{k+1}(\x)=\mathcal{T}_{k+1}(\x,\y_{k}(\x)), \ k=0,\cdots,K-1,\label{eq:update-t}
\end{equation}
where $\mathcal{T}_k(\x,\cdot)$ stands for a schematic iterative module originated from a certain simple bi-level solution strategy on Eq.~\eqref{eq:simple-blp} with a fixed UL variable $\x$.\footnote{It can be seen that $\T_k$ actually should integrate the information from both the UL and LL subproblems in Eqs.~\eqref{eq:blp}-\eqref{eq:lower-level}. We will discuss specific choices of $\T_k$ in the following subsection.} Let $\y_0=\T_0(\x)$ be the initialization of the above scheme and denote
$\y_K(\x)$ as the output of Eq.~\eqref{eq:update-t} after $K$ iterations (including the initial calculation $\T_0$). Then we can replace $\varphi(\x)$ by $F(\mathbf{x},\mathbf{y}_K(\x))$ and obtain the following approximation of  Eq.~\eqref{eq:oblp}:
\begin{equation}
\min\limits_{\mathbf{x}\in\X} \varphi_K(\x)= F(\mathbf{x},\mathbf{y}_K(\x)).\label{eq:upper_varphiK}
\end{equation}
With the above procedure, the BLPs in Eqs.~\eqref{eq:blp}-\eqref{eq:lower-level} is approximated by a sequence of standard single-level optimization problems. For each approximation subproblem in Eq.~\eqref{eq:upper_varphiK}, its descent direction is actually implicitly representable in terms of a certain simple bi-level solution strategy (i.e., Eq.~\eqref{eq:update-t}). Therefore, these existing automatic differentiation techniques all can be involved to achieve optimal solutions to Eq.~\eqref{eq:upper_varphiK}~\cite{franceschi2017forward,baydin2017automatic}. 

\subsection{Aggregated Iteration Modules}

Now optimizing BLPs in Eqs.~\eqref{eq:blp}-\eqref{eq:lower-level} reduces to the problem of designing proper $\T_k$ for Eq.~\eqref{eq:update-t}. As discussed above, $\T_k$ is related to both the UL and LL objectives. So it is natural to aggregate the descent information of these two subproblems to design $\T_k$. Specifically, for a given $\x$, the descent directions of the UL and LL objectives can be defined as
$$
\begin{array}{l}
\mathbf{d}^{{F}}_k(\mathbf{x})=s_u\nabla_{\mathbf{y}} F(\mathbf{x},\mathbf{y}_{k}),\\
\mathbf{d}^{{f}}_k(\mathbf{x})=s_l\nabla_{\mathbf{y}} f(\mathbf{x},\mathbf{y}_{k}),
\end{array}
$$
where $s_u,s_l>0$ are their step size parameters. Then we formulate $\T_k$ as the following first-order descent scheme:
\begin{equation}
\T_{k+1}\left(\x,\y_k(\x)\right)=\y_k -\left( \alpha_k\mathbf{d}^{{F}}_k(\mathbf{x})+(1-\alpha_k)\mathbf{d}^{{f}}_k(\mathbf{x})\right),\label{eq:lower}
\end{equation}
where $\alpha_k\in(0,1)$ denotes the aggregation parameter. 
\begin{remark}
	In this part, we just introduce a gradient aggregation based $\T_k$ to handle the simple bi-level subproblem in Eq.~\eqref{eq:simple-blp}. Indeed, our theoretical analysis in Section~\ref{sec:theory} will demonstrate that BDA algorithmic framework is flexible enough to incorporate a variety of numerical schemes. For example, in Supplemental Material, we also design an appropriate $\T_k$ to handle BLPs with nonsmooth LL objective while its convergence is still strictly guaranteed within our framework. 
\end{remark}

\section{Theoretical Investigations}\label{sec:theory}

In this section, we first derive a general convergence proof recipe together with two elementary properties to systematically investigate the convergence behaviors of bi-level FOMs (Section~\ref{subsec:recipe}). Following this roadmap, the convergence of our BDA can successfully {get rid of depending upon the LLS condition} (Section~\ref{subsec:convergence}). We also improve the convergence results for the existing bi-level FOMs in the LLS scenario (Section~\ref{subsec:revist}). To avoid triviality, hereafter we always assume that $\S(\x)$ is nonempty for any $\x \in \X$. Please notice that all the proofs of our theoretical results are stated in the Supplemental Material. 

\subsection{A General Proof Recipe}\label{subsec:recipe}

We first state some definitions, which are necessary for our analysis.\footnote{Please also refer to~\cite{rockafellar2009variational} for more details on these variational analysis properties.} 
A series of continuity properties for set-valued mappings and functions can be defined as follows.
\begin{defi}\label{defnitionconti}
	A set-valued mapping $\S(\x) : \mathbb{R}^n \rightrightarrows \mathbb{R}^m$ is Outer Semi-Continuous (OSC) at $\bar{\mathbf{x}}$ if $\limsup_{\mathbf{x} \rightarrow \bar{\mathbf{x}}}\S(\mathbf{x}) \subseteq \S(\bar{\mathbf{x}})$ and Inner Semi-Continuous (ISC) at $\bar{\mathbf{x}}$ if $\liminf_{\mathbf{x} \rightarrow \bar{\mathbf{x}}}\S(\mathbf{x}) \supseteq \S(\bar{\mathbf{x}}).$ $\S(\x)$ is called continuous at $\bar{\mathbf{x}}$ if it is both OSC and ISC at $\bar{\mathbf{x}}$, as expressed by $\lim_{\mathbf{x} \rightarrow \bar{\mathbf{x}}}\S(\mathbf{x}) = \S(\bar{\mathbf{x}})$. Here $\limsup_{\mathbf{x} \rightarrow \bar{\mathbf{x}}}\S(\mathbf{x})$ and $\liminf_{\mathbf{x} \rightarrow \bar{\mathbf{x}}}\S(\mathbf{x})$ are defined as 
\begin{equation*}
\begin{array}{l}
	\limsup\limits_{\mathbf{x} \rightarrow \bar{\mathbf{x}}}\S(\mathbf{x})= \left\{\mathbf{y} |~ \exists \mathbf{x}^{\nu} \rightarrow \bar{\mathbf{x}}, \exists \mathbf{y}^{\nu} \rightarrow \mathbf{y}, \mathbf{y}^{\nu} \in \S(\mathbf{x}^{\nu}) \right\},\\
	\liminf\limits_{\mathbf{x} \rightarrow \bar{\mathbf{x}}}\S(\mathbf{x})=\left\{\mathbf{y} |~ \forall \mathbf{x}^{\nu}\rightarrow \bar{\mathbf{x}}, \exists \mathbf{y}^{\nu} \rightarrow \mathbf{y}, \mathbf{y}^{\nu} \in \S(\mathbf{x}^{\nu}) \right\},
	\end{array}
\end{equation*}
where $\nu\in\mathbb{N}$.
\end{defi}
\begin{defi}
	A function $\varphi(\x):\mathbb{R}^n\to\mathbb{R}$ is Upper Semi-Continuous (USC) at $\bar{\x}$ if 
	$\limsup_{\x \rightarrow \bar{\x}}\varphi(\x)\leq \varphi(\bar{\x})$, or equivalently $\limsup_{\x \rightarrow \bar{\x}}\varphi(\x)=\varphi(\bar{\x}),$ 
	and USC on $\mathbb{R}^n$ if this holds for every $\bar{\x}\in\mathbb{R}^n$. Similarly, $\varphi(\x)$ is Lower Semi-Continuous (LSC) at $\bar{\x}$ if $\liminf_{\x \rightarrow \bar{\x}}\varphi(\x)\geq \varphi(\bar{\x})$, or equivalently $\liminf_{\x \rightarrow \bar{\x}}\varphi(\x)=\varphi(\bar{\x}),$ 
	and LSC on $\mathbb{R}^n$ if this holds for every $\bar{\x}\in\mathbb{R}^n$. Here $\limsup_{\x \rightarrow \bar{\x}} \varphi(\x)$ and $\liminf_{\x \rightarrow \bar{\x}} \varphi(\x)$ are respectively defined as
	\begin{equation*}
		\begin{array}{l}
		\limsup\limits_{\x \rightarrow \bar{\x}} \varphi(\x) =\lim\limits_{\delta\to 0}\left[\sup_{\x \in \mathbb{B}_{\delta}(\bar{\x})} \varphi(\x) \right]
		\end{array}
	\end{equation*}
	and 
	\begin{equation*}
	\begin{array}{l}
		\liminf\limits_{\x \rightarrow \bar{\x}} \varphi(\x) =\lim\limits_{\delta\to 0}\left[ \inf_{\x \in \mathbb{B}_{\delta}(\bar{\x})} \varphi(\x) \right]
		\end{array}
	\end{equation*}
	where $\mathbb{B}_{\delta}(\bar{\x})=\{\x\big|\|\x-\bar{\x}\|\leq\delta\}$.
\end{defi}
Then for a given function $f(\x,\y)$, we state the property that it is level-bounded in $\x$ locally uniform in $\y$ in the following definition.
\begin{defi}\label{def:uniform_boundness}
	Given a function $f(\x,\y) : \mathbb{R}^n \times \mathbb{R}^m \rightarrow \mathbb{R}$, if for a point $\bar{\x} \in \X\subseteq\mathbb{R}^n$ and $c \in \mathbb{R}$, there exist $\delta > 0$ along with a bounded set $\mathcal{B} \in \mathbb{R}^m$, such that
	$$
	\left\{\y \in \mathbb{R}^m ~|~ f(\x,\y) \le c \right\} \subseteq \mathcal{B}, \ \forall \x \in \mathbb{B}_{\delta}(\bar{\x}) \cap \X,
	$$
	then we call $f(\x,\y)$ is level-bounded in $\y$ locally uniformly in $\bar{\x} \in \X$. If the above property holds for each $\bar{\x} \in \X$, we further call $f(\x,\y)$ level-bounded in $\y$ locally uniformly in $\x\in\X$.
\end{defi}

Now we are ready to establish the general proof recipe, which describes the main steps to achieve the converge guarantees for our bi-level updating scheme (stated in Eqs.~\eqref{eq:update-t}-\eqref{eq:upper_varphiK}, with a schematic $\T_k$). Basically, our proof methodology consists of two main steps:
\begin{enumerate}
	\item[(1)] \textbf{LL solution set property:} For any $\epsilon>0$, there exists $k(\epsilon)>0$ such that whenever $K>k(\epsilon)$, 	
	$$
	\sup_{\mathbf{x}\in\mathcal{X}}\mathtt{dist}(\mathbf{y}_{K}(\mathbf{x}),\S(\x))\leq\epsilon.\label{eq:dist-y-s}
	$$	
	\item[(2)] \textbf{UL objective convergence property:} $\varphi(\x)$ is LSC 
	on $\mathcal{X}$, and for each $\x \in \mathcal{X}$,
	$$\lim\limits_{K \rightarrow \infty}\varphi_K(\x) \rightarrow \varphi(\x).\label{eq:dist-varphi}$$
\end{enumerate}

Equipped with the above two properties, we can establish our general convergence results in the following theorem for the schematic bi-level scheme in Eqs.~\eqref{eq:update-t}-\eqref{eq:upper_varphiK}. 
\begin{thm}\label{thm:general}
	Suppose both the above LL solution set and UL objective convergence properties hold. Then we have
	\begin{itemize}
		\item[(1)]  suppose $\x_K\in\arg\min_{\x\in\X}\varphi_{K}(\x)$, then any limit point $\bar{\x}$ of the sequence $\{\x_K\}$ satisfies that $\bar{\x}\in\arg\min_{\x\in\X}\varphi(\x)$. 
		\item[(2)] $\inf_{\x \in \X}\varphi_K(\x) \rightarrow \inf_{\x \in \X} \varphi(\x)$ as $K \rightarrow \infty$. 
	\end{itemize} 
\end{thm}
\begin{remark}
Indeed, if $\x_K$ is a local minimum of $\varphi_{K}(\x)$ with uniform neighborhood modulus $\delta > 0$, we can still have that any limit point $\bar{\x}$ of the sequence $\{\x_K\}$ is a local minimum of $\varphi(\x)$. Please see our Supplemental Material for more details on this issue.
\end{remark}

\subsection{Convergence Properties of BDA}\label{subsec:convergence}

The objective here is to demonstrate that our BDA meets these two elementary properties required by Theorem~\ref{thm:general}. Before proving the convergence results for BDA, we first take the following as our blanket assumption. 
\begin{assum}\label{assum:F}
	For any $\x \in \X$, $F(\x,\cdot) : \mathbb{R}^m \rightarrow \mathbb{R}$ is $L_0$-Lipschitz continuous, $L_F$-smooth, and $\sigma$-strongly convex, $f(\x,\cdot) : \mathbb{R}^m \rightarrow \mathbb{R}$ is $L_f$-smooth and convex.
\end{assum}

Please notice that Assumption~\ref{assum:F} is quite standard for BLPs in machine learning areas~\cite{franceschi2018bilevel,shaban2018truncated}. As can be seen, it is satisfied for all the applications considered in this work. We first present some necessary variational analysis preliminaries. Denoting 
\begin{equation*}
\tilde{\S}(\x) = \arg\min\limits_{\y \in \S(\x) } F(\x,\y),
\end{equation*}
under Assumption~\ref{assum:F}, we can quickly obtain that $\tilde{\S}(\x)$ is nonempty and unique for any $\x \in \X$. Moreover, we can derive the boundedness of $\tilde{\S}(\x)$ in the following lemma.
\begin{lemma}\label{lemma:tildeS_bounded}
	Suppose $F(\x,\y)$ is level-bounded in $\y$ locally uniformly in $\x\in\X$. If $\S(\x)$ is ISC on $\X$, then $\cup_{\x\in \X} \tilde{\S}(\x)$ is bounded. 
\end{lemma}
Thanks to the continuity of $f(\x,\y)$, we further have the following result.
\begin{lemma}\label{lemma:f_USC}
Denote $f^{\ast}(\x)=\min_{\y}f(\x,\y)$. If $f(\x,\y)$ is continuous on $\X\times\mathbb{R}^m$, then $f^{\ast}(\x)$ is USC on $\X$. 
\end{lemma}

Now we are ready to establish our fundamental LL solution set and UL objective convergence properties required in Theorem~\ref{thm:general}. In the following proposition, we first derive the convergence of $\{\y_K(\x)\}$ in the light of the general fact stated in~\cite{sabach2017first}.
\begin{prop}\label{prop:f_conver}
	Suppose Assumption \ref{assum:F} is satisfied and let $\{\y_K\}$ be defined as in Eq.~\eqref{eq:lower}, $s_l \in (0,1/L_f]$, $s_u \in (0,2/(L_F+\sigma)]$, $$\alpha_k = \min \left\{2\gamma/k(1-\beta),1-\varepsilon \right\},$$ with $k \ge 1$, $\varepsilon>0$,  $\gamma \in (0,1]$, and 
	$$\beta = \sqrt{1-2s_u\sigma L_F/(\sigma + L_F)}.$$ 
	Denote 
	$\tilde{\y}_K(\x) = \y_K(\x)-s_l\nabla_{\y}f(\x,\y_K(\x)),$ 
	and 
	$$C_{\y^*(\x)} = \max \left\{ \|\y_0 - \y^*(\x)\|, \frac{s_u}{1-\beta}\|\nabla_\y F(\x,\y^*(\x))\|  \right\},$$ 
	with $\y^*(\x) \in \tilde{\S}(\x)$ and $\x \in \X$. Then we have
	\begin{align*}
	\|\y_K(\x) - \y^*(\x)\| &\le C_{\y^*(\x)},\\
	\|\y_K(\x) - \tilde{\y}_K(\x)\| &\le \frac{2C_{\y^*(\x)}(J+2)}{K(1-\beta)},\\
	f(\x,\tilde{\y}_K(\x)) - f^*(\x) &\le \frac{2C_{\y^*(\x)}^2(J+2)}{K(1-\beta)s_l},
	\end{align*}
	where $J = \lfloor 2/(1-\beta) \rfloor$. Furthermore, for any $\x \in \X$, $\{\y_K(\x)\}$ converges to $\tilde{\S}(\x)$ as $K \rightarrow \infty$. 
\end{prop}
Proposition~\ref{prop:f_conver}, together with Lemma~\ref{lemma:tildeS_bounded}, shows that $\{\tilde{\y}_K(\x)\}$ is a bounded sequence and $\{f(\x,\tilde{\y}_K(\x))\}$ uniformly converges. We next prove the uniform convergence of $\{\tilde{\y}_K(\x)\}$ towards the solution set $\mathcal{S}(\mathbf{x})$ through the uniform convergence of $\{f(\x,\tilde{\y}_K(\x))\}$.
\begin{prop}\label{prop:yK_conver}
	Let $\mathcal{Y} \subseteq \mathbb{R}^m$ be a bounded set and $\epsilon > 0$. If $\S(\x)$ is ISC on $\X$, then there exists $\delta > 0$ such that for any $\y \in \mathcal{Y}$, $$\sup_{\x \in \X} \mathrm{dist}(\y,\S(\x)) \le \epsilon,$$ in case $\sup_{\x \in \X}\left\{ f(\x,\y) - f^{\ast}(\x) \right\} \le \delta$ is satisfied.
\end{prop}
Combining Lemmas~\ref{lemma:tildeS_bounded} and~\ref{lemma:f_USC}, together with Proposition~\ref{prop:yK_conver}, the \emph{LL solution set} property required in Theorem~\ref{thm:general} can be eventually derived. Let us now prove the LSC property of $\varphi$ on $\X$ in the following proposition.
\begin{prop}\label{prop:varphi_LSC}
	Suppose $F(\x,\y)$ is level-bounded in $\y$ locally uniformly in $\x\in\X$. If $\S(\x)$ is OSC at ${\x}\in\X$, then $\varphi(\x)$ is LSC at ${\x}\in\X$. 
\end{prop}
Then the \emph{UL objective convergence} property required in Theorem~\ref{thm:general} can be obtained subsequently based on Proposition~\ref{prop:varphi_LSC}, In summary, we present the main convergence results of BDA in the following theorem. 
\begin{thm}\label{thm:conver_BDA}
	Suppose Assumption~\ref{assum:F} is satisfied and let $\{\y_K\}$ be defined as in Eq.~\eqref{eq:lower}, $s_l \in (0,1/L_f]$, $s_u \in (0,2/(L_F+\sigma)]$, $$\alpha_k = \min \left\{2\gamma/k(1-\beta),1-\varepsilon \right\},$$ with $k \ge 1$, $\varepsilon>0$,  $\gamma \in (0,1]$, and $$\beta = \sqrt{1-2s_u\sigma L_F/(\sigma + L_F)}.$$ Assume further that $\S(\x)$ is continuous on $\X$. Then we have that both the LL solution set and UL objective convergence properties hold. 
\end{thm}
\begin{remark}
	Our proposed theoretical results are indeed general enough for BLPs in different application scenarios. For example, when the LL objective takes a nonsmooth form, e.g., $h=f + g$ with smooth $f$ and nonsmooth $g$, we can adopt the proximal operation based iteration module \cite{beck2017first} to construct $\T_k$ within our BDA framework. The convergence proofs are highly similar to that in Theorem~\ref{thm:conver_BDA}. More details on such extension can be found in our Supplemental Material.
\end{remark}

\subsection{Improving Existing LLS Theories}\label{subsec:revist}

Although with the LLS simplification on BLPs in Eqs.~\eqref{eq:blp}-\eqref{eq:lower-level}, the theoretical properties of the existing bi-level FOMs are still not very convincing. Their convergence proofs in essence depend on the strong convexity of the LL objective, which may restrict the use of these approaches in complex machine learning applications. In this subsection, by weakening the required assumptions, we improve the convergence results in~\cite{franceschi2018bilevel,shaban2018truncated} for these conventional bi-level FOMs in the LLS scenario. Specifically, we first introduce an assumption on the LL objective $f(\x,\y)$, which is needed for our analysis in this subsection.
\begin{assum}\label{assum:f} 
	$f(\x,\y) : \mathbb{R}^n \times \mathbb{R}^m \rightarrow \mathbb{R}$ is level-bounded in $\y$ locally uniformly in $\x\in\X$.
\end{assum}
In fact, Assumption~\ref{assum:f} is mild and satisfied by a large number of bi-level FOMs, when the LL subproblem is convex but not necessarily strongly convex. In contrast, theoretical results in existing literature~\cite{franceschi2018bilevel,shaban2018truncated} require the more restrictive (local) strong convexity property on the LL objective to meets the LLS condition. 

Under Assumption~\ref{assum:f}, the following lemma verifies the continuity of $\S(\x)$ in the LLS scenario.
\begin{lemma}\label{prop:revisting_continuous_S}
	Suppose  $\S(\x)$ is single-valued on $\X$ and Assumption~\ref{assum:f} is satisfied. Then $\S(\x)$ is continuous on $\X$.
\end{lemma}
As can be seen from the proof of Theorem~\ref{thm:revisting_conver} in our Supplemental Material, Lemma \ref{prop:revisting_continuous_S} and the uniform convergence of $\{f(\x,\y_K(\x))\}$ actually imply the \emph{LL solution set} and \emph{UL objective convergence} properties. Hence Theorem~\ref{thm:general} is applicable, which inspires an improved version of the convergence results for the existing bi-level FOMs as follows.
\begin{thm}\label{thm:revisting_conver}
	Suppose  $\S(\x)$ is single-valued on $\X$ and Assumption~\ref{assum:f} is satisfied, $\{\y_K(\x)\}$ is uniformly bounded on $\X$, and  $\{f(\x,\y_K(\x))\}$ converges uniformly to $f^*(\x)$ on $\X$ as $K \rightarrow \infty$. Then we have that both the LL solution set and UL objective convergence properties hold. 
\end{thm}
\begin{remark}
Theorem~\ref{thm:revisting_conver} actually improves the converge results in \cite{franceschi2018bilevel}. In fact, the uniform convergence assumption of $\{\y_K(\x)\}$ towards $\y^*(\x)$ required in \cite{franceschi2018bilevel} is essentially based on the strong convexity assumption (see Remark 3.3 of \cite{franceschi2018bilevel}). Instead of assuming such strong convexity, we only need to assume a weaker condition that $\{f(\x,\y_K(\x))\}$ converges uniformly to $f^*(\x)$ on $\X$ as $K \rightarrow \infty$. 
\end{remark}

It is natural for us to illustrate our improvement in terms of concrete applications. Specifically, we take the gradient-based bi-level scheme summarized in Section~\ref{sec:rhg} (which has been used in \cite{franceschi2018bilevel,jenni2018deep,shaban2018truncated,zugner2019adversarial,rajeswaran2019meta}). In the following two propositions, we assume that $f(\x,\cdot):\mathbb{R}^m\to\mathbb{R}$ is $L_f$-smooth and convex, and $s_l \leq 1/L_f$. 

Inspired by Theorems 10.21 and 10.23 in \cite{beck2017first}, we first derive the following proposition. 
\begin{prop}\label{pg}
	Let $\{\y_K\}$ be defined as in Eq.~\eqref{eq:gradient_f}. Then it holds that $$\|\y_K(\x) - \y^*(\x)\| \le \|\y_0 - \y^*(\x)\|,$$ and $$f(\y_K(\x)) - f^*(\x) \le \frac{\|\y_0-\y^*(\x)\|^2}{2s_lK},$$ with $\y^*(\x) \in \S(\x)$ and $\x \in \X$.
\end{prop}

Then in the following proposition we can immediately verify our required assumption on $\{f(\x,\y_K(\x))\}$ in the absence of the strong convexity property on the LL objective.
\begin{prop}\label{prop:fK-conver}
	Suppose that $\S(\x)$ is single-valued on $\X$ and Assumption~\ref{assum:f} is satisfied. Let $\{\y_K\}$ be defined as in Eq.~\eqref{eq:gradient_f}. Then $\{\y_K(\x)\}$ is uniformly bounded on $\X$ and $\{f(\x,\y_K(\x))\}$ converges uniformly to $f^*(\x)$ on $\X$ as $K \rightarrow \infty$.
\end{prop}

\begin{remark}
	When the LL subproblem is convex, but not necessarily strongly convex, a large number of gradient-based methods, including accelerated gradient methods such as FISTA \cite{beck2009fast} and block coordinate descent method \cite{tseng2001convergence}, automatically meet our assumption, i.e., the uniform convergence of optimal values $\{f(\x,\y_K(\x))\}$ towards $f^*(\x)$ on $\X$.
\end{remark}

\begin{figure*}[t]
	\centering \begin{tabular}{c@{\extracolsep{0.2em}}c@{\extracolsep{0.2em}}c@{\extracolsep{0.2em}}c}
		\includegraphics[width=0.225\textwidth]{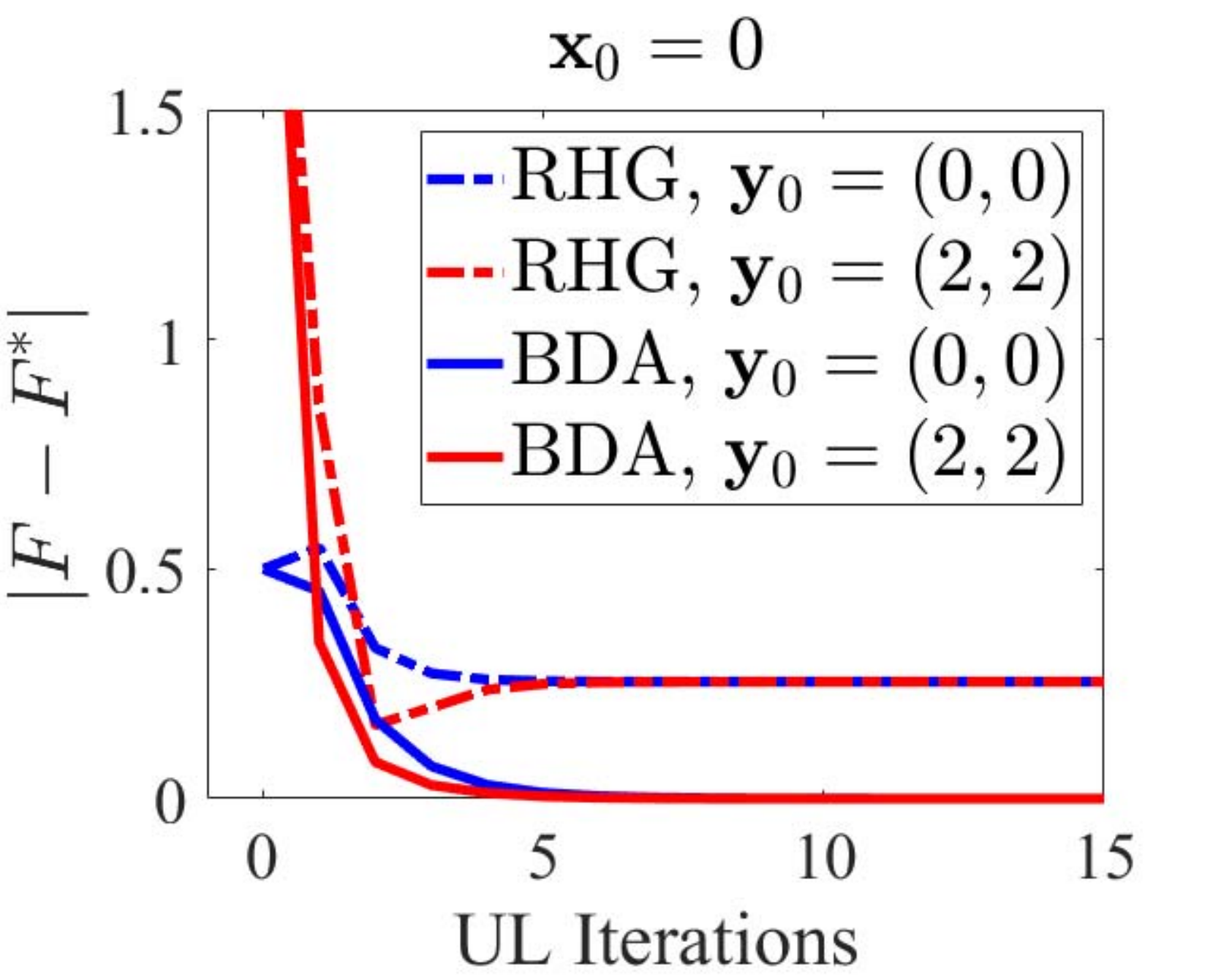}
		&\includegraphics[width=0.225\textwidth]{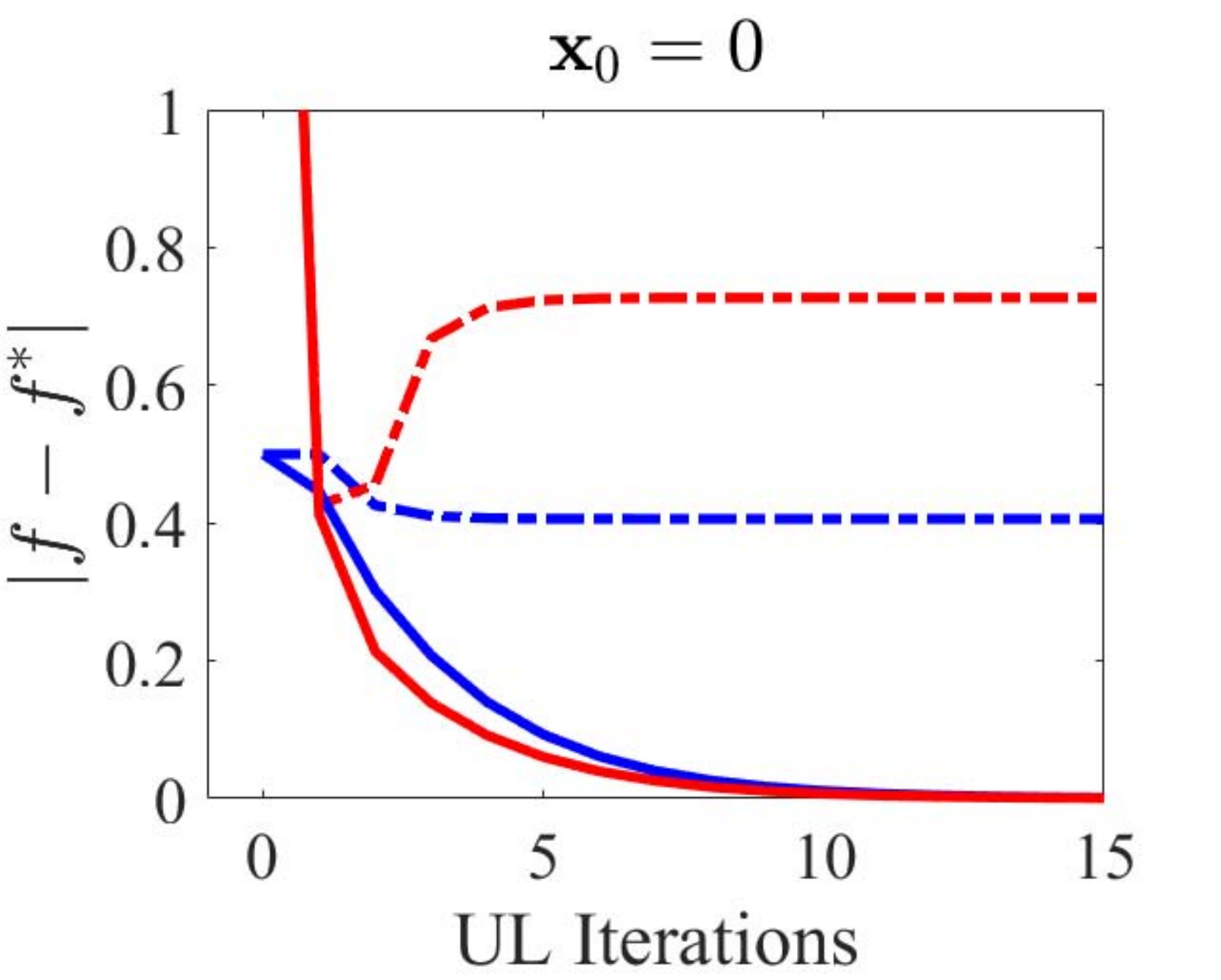}
		&\includegraphics[width=0.225\textwidth]{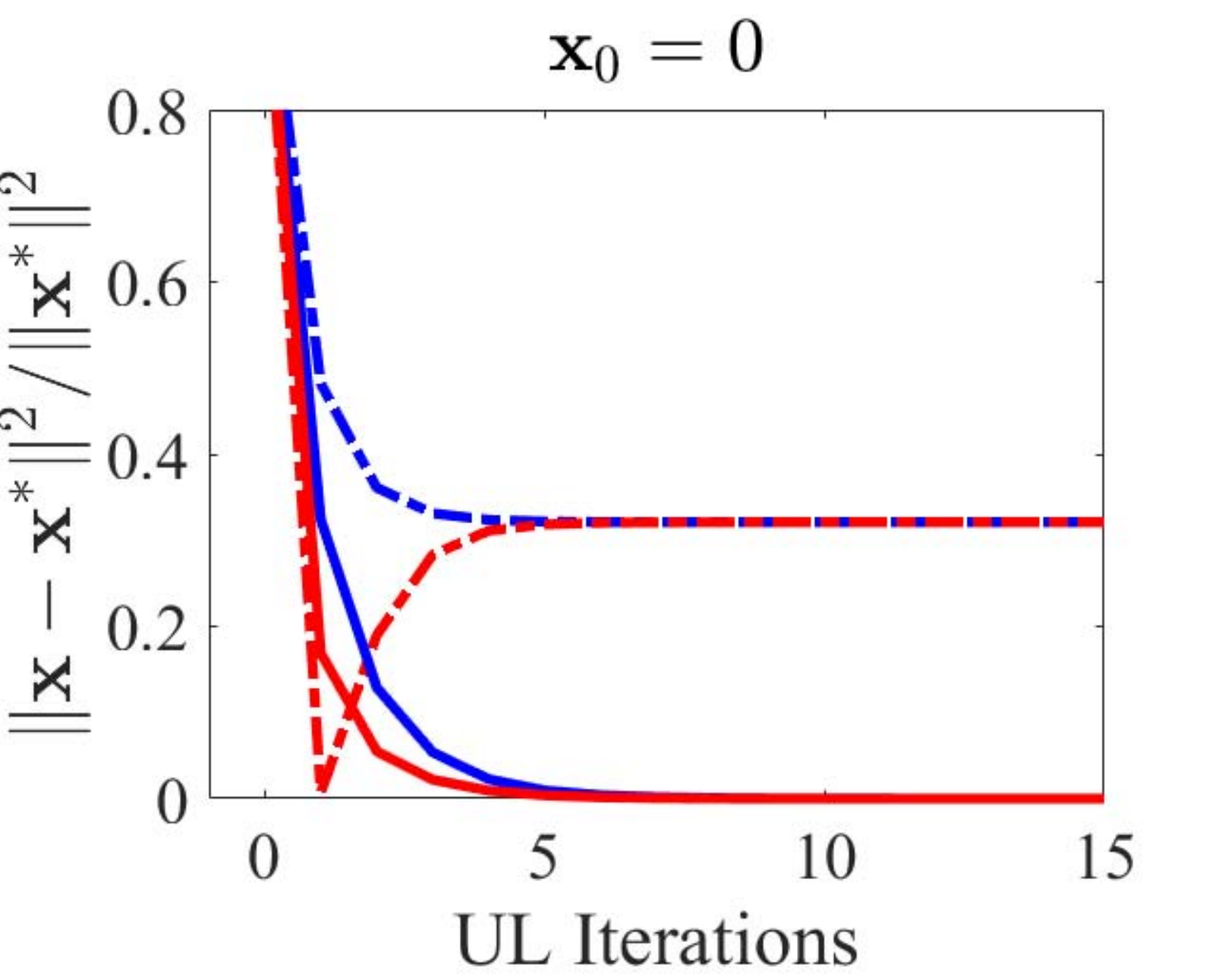}
		&\includegraphics[width=0.225\textwidth]{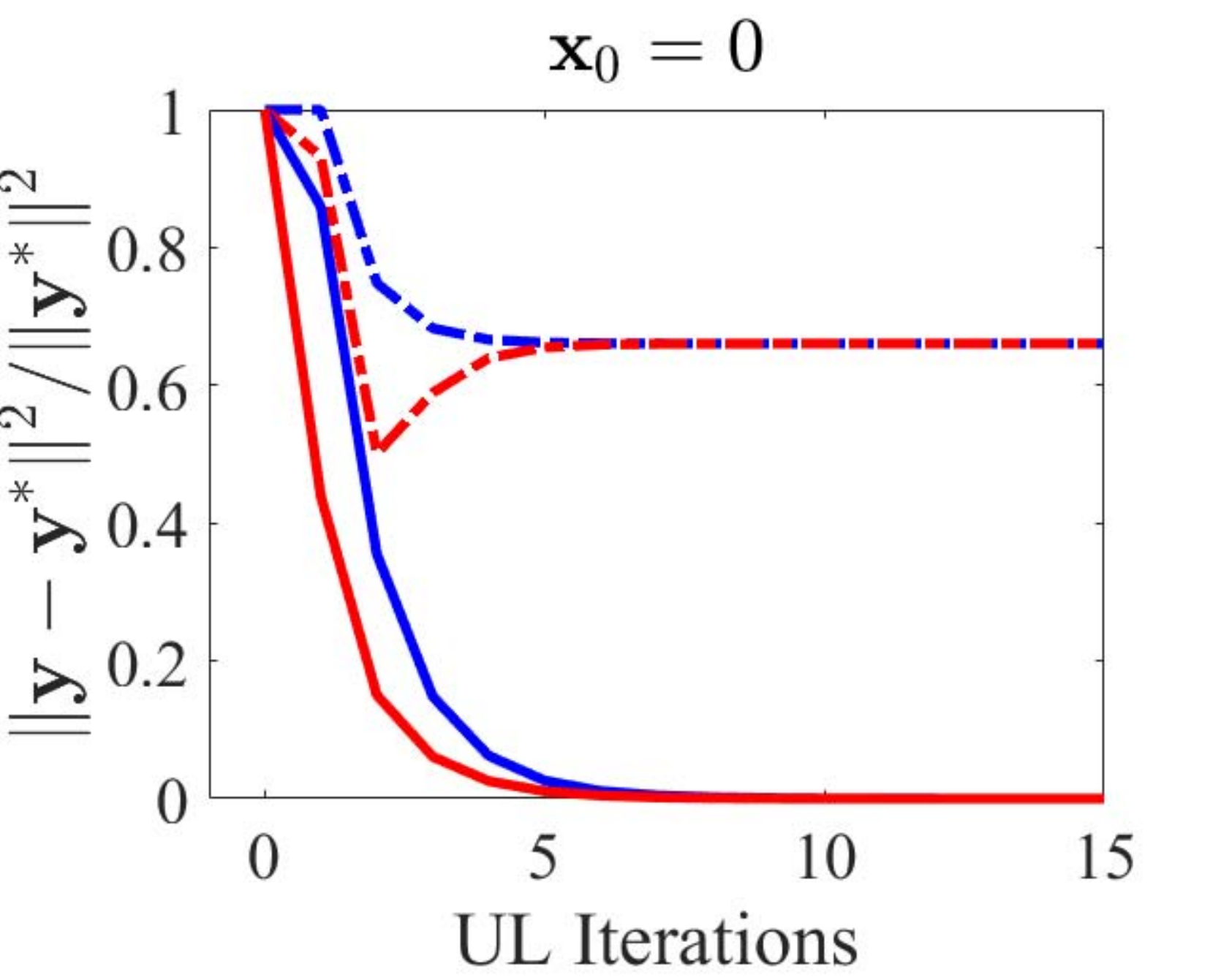}\\
		\includegraphics[width=0.225\textwidth]{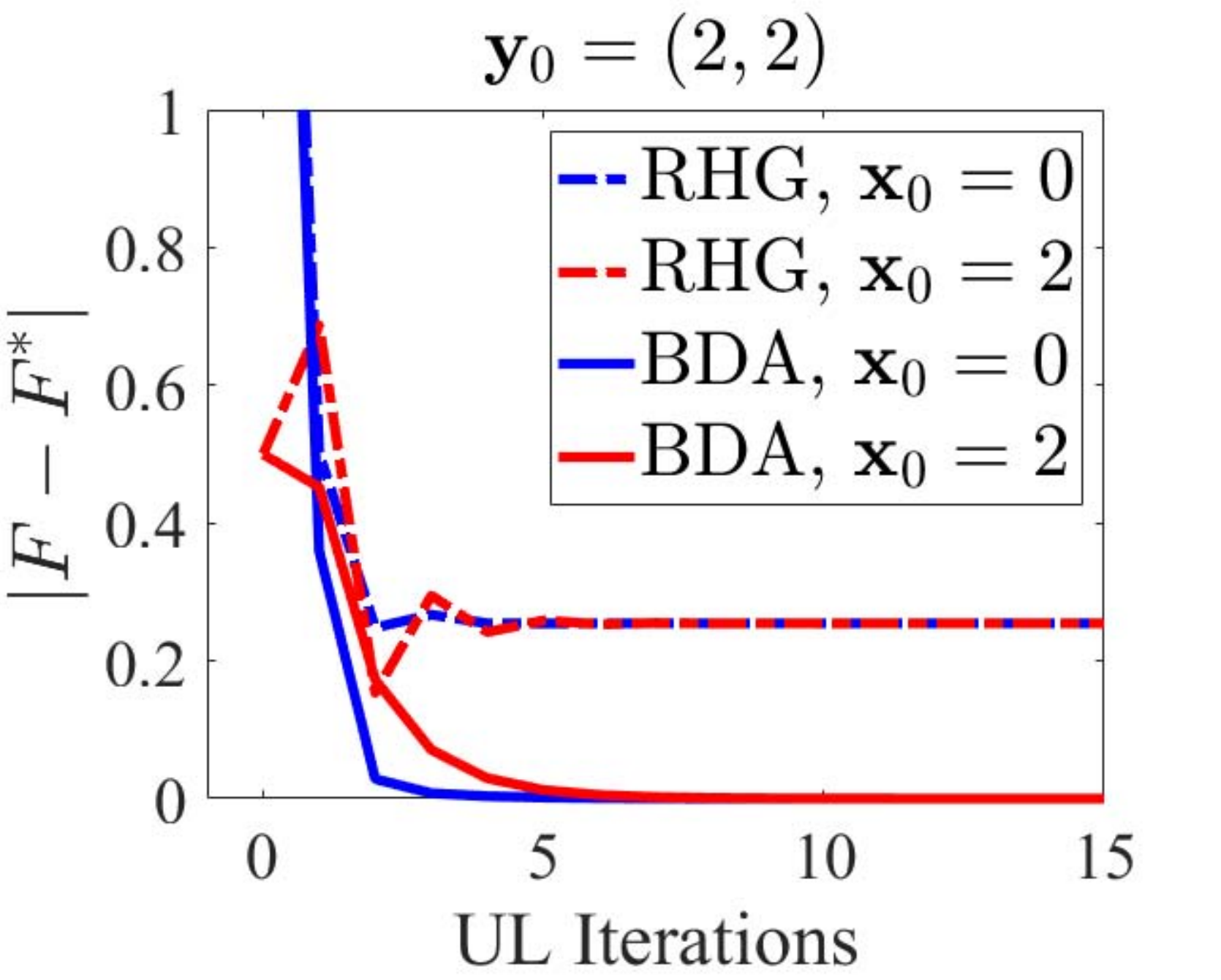}
		&\includegraphics[width=0.225\textwidth]{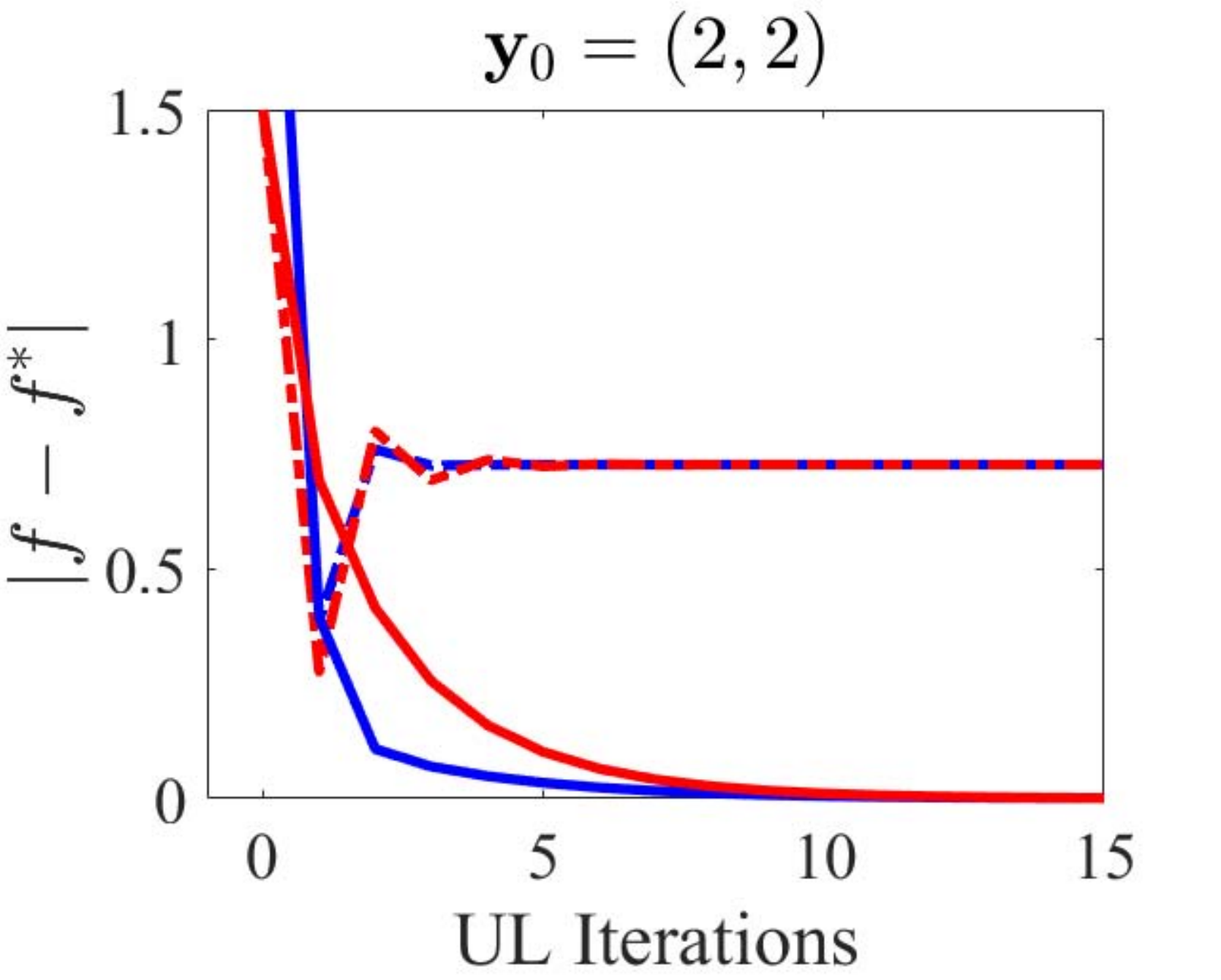}
		&\includegraphics[width=0.225\textwidth]{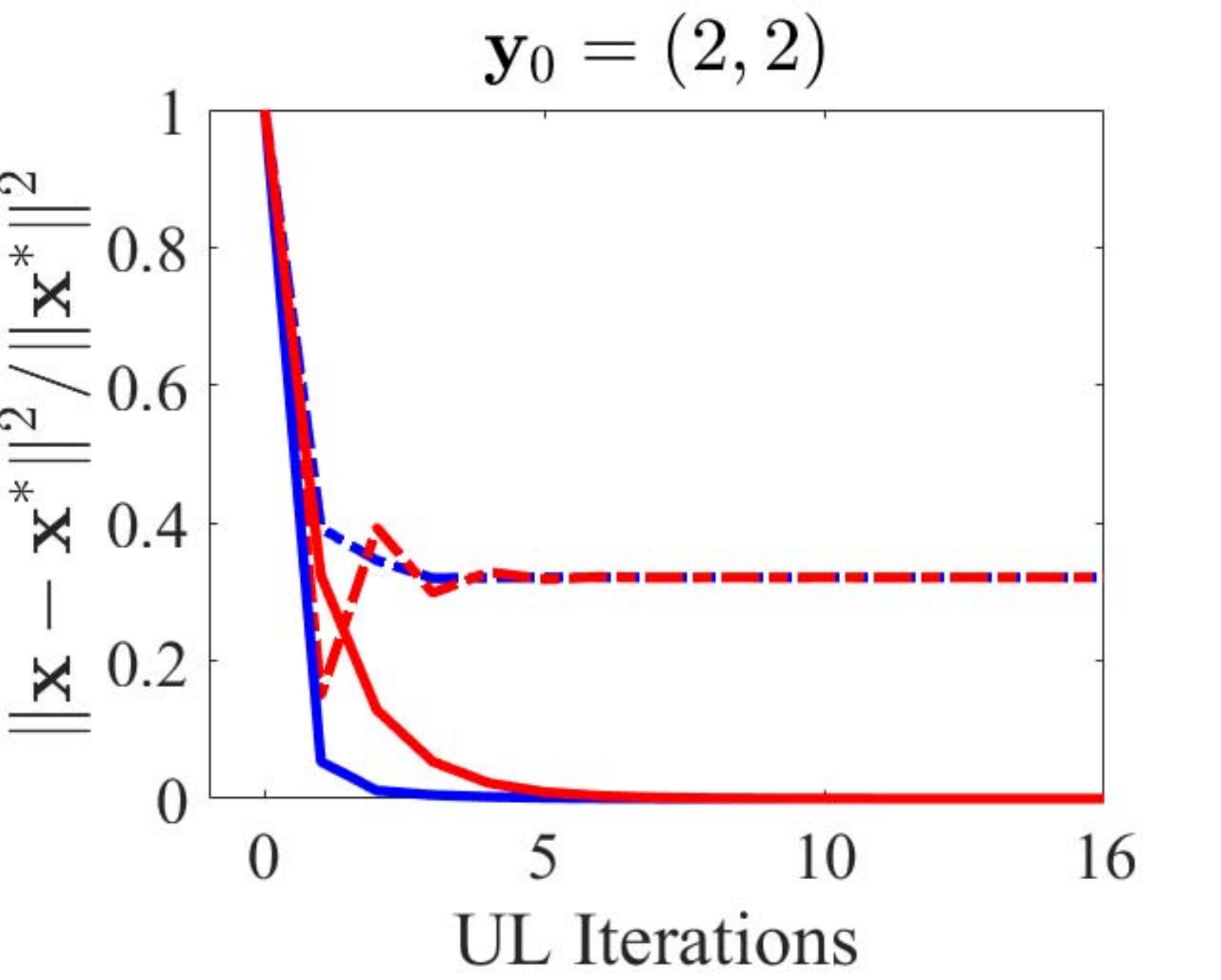}
		&\includegraphics[width=0.225\textwidth]{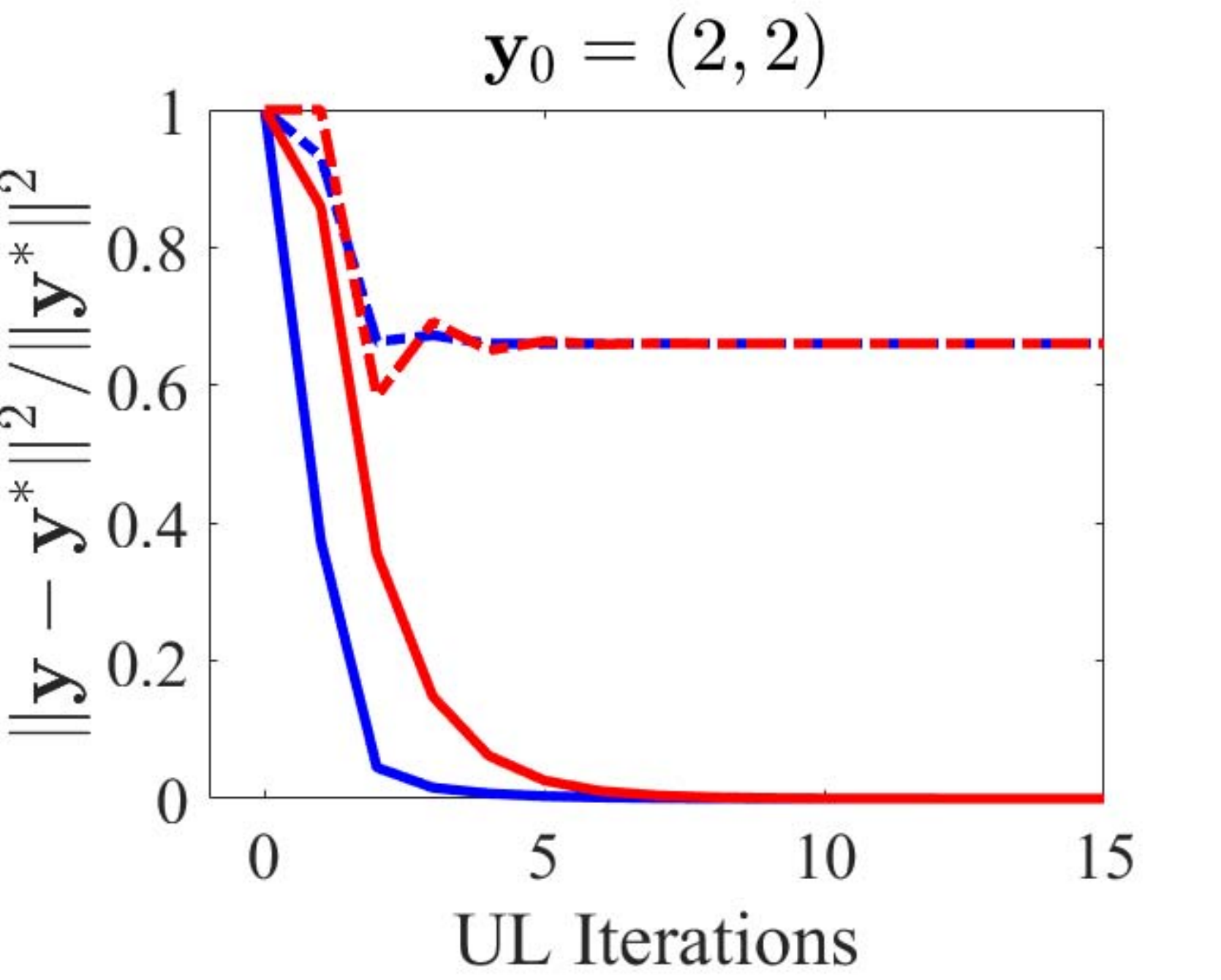}\\	
	\end{tabular}
	\caption{Illustrating the numerical performance of first-order BLPs algorithms with different initialization points. Top row: fix $\x_0=0$ and vary $\y_0=(0,0), (2,2)$. Bottom row: fix $\y_0=(2,2)$ and vary $\x_0=0, 2$. We fix $K=16$ for UL iterations. The dashed and solid curves denote the results of RHG and BDA, respectively. The legend is only plotted in the first subfigure.} \label{fig:toy_initial}
\end{figure*}

\begin{figure*}[t]
	\centering \begin{tabular}{c@{\extracolsep{0.2em}}c@{\extracolsep{0.2em}}c@{\extracolsep{0.2em}}c}
		\includegraphics[width=0.225\textwidth]{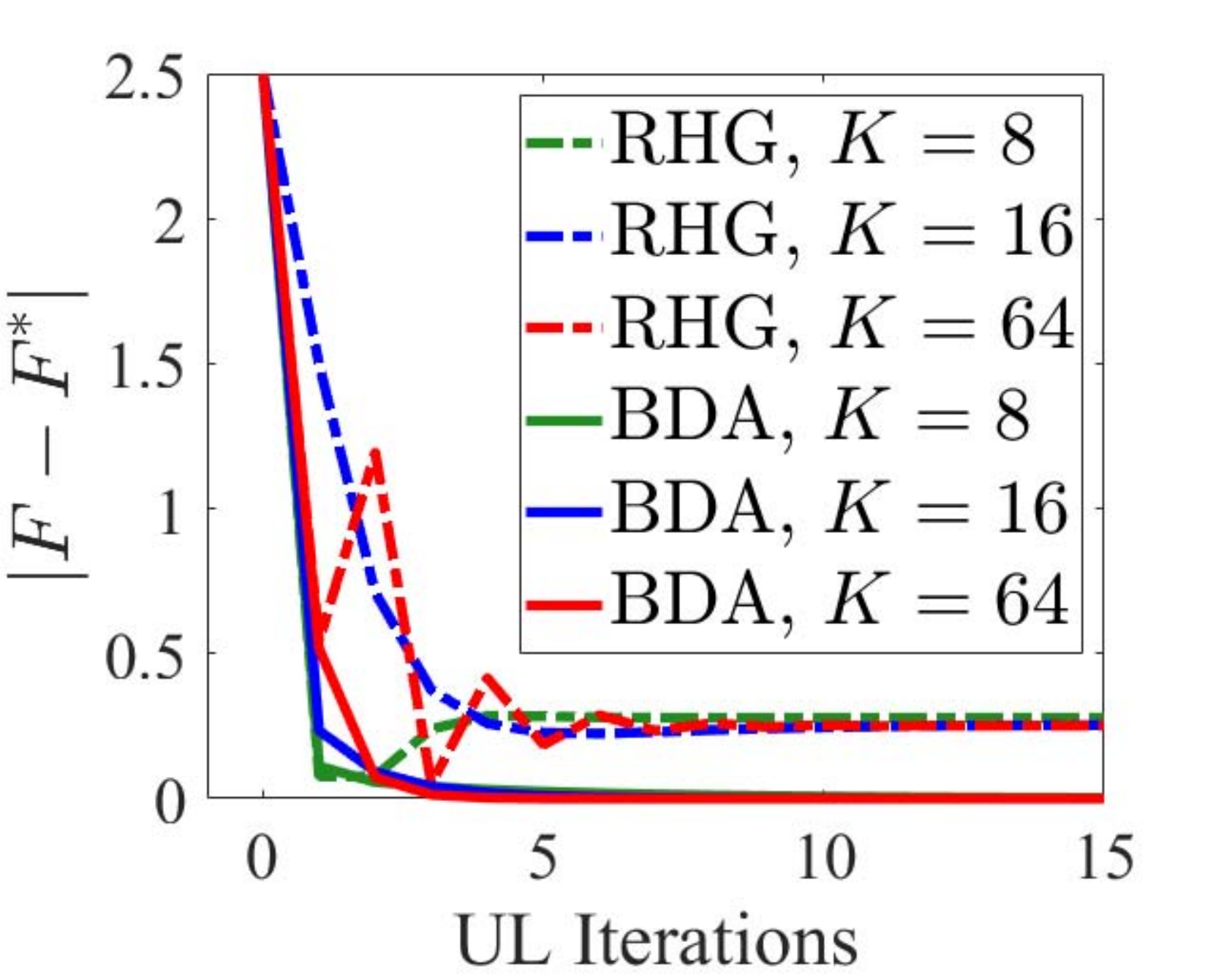}
		& \includegraphics[width=0.225\textwidth]{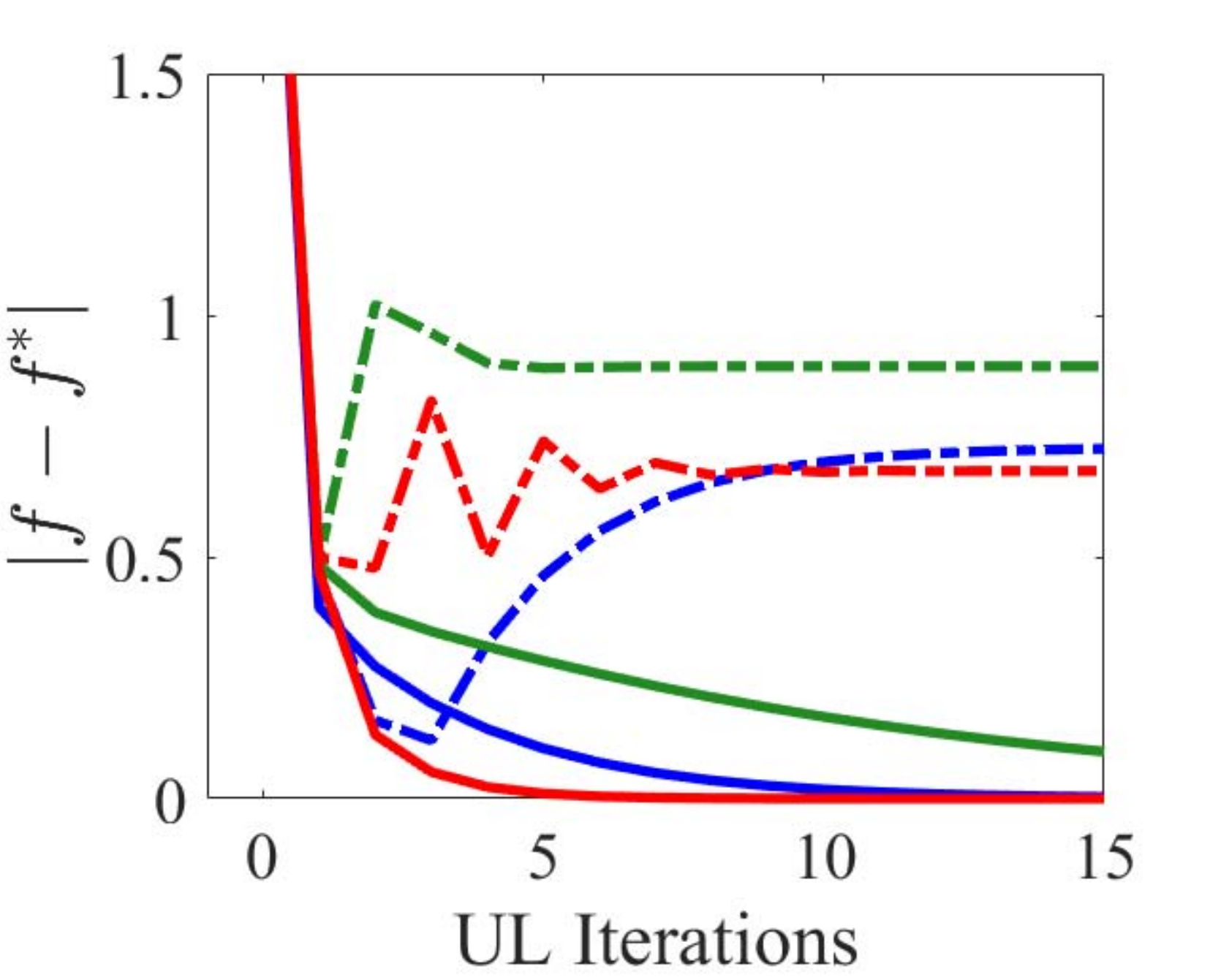}
		& \includegraphics[width=0.225\textwidth]{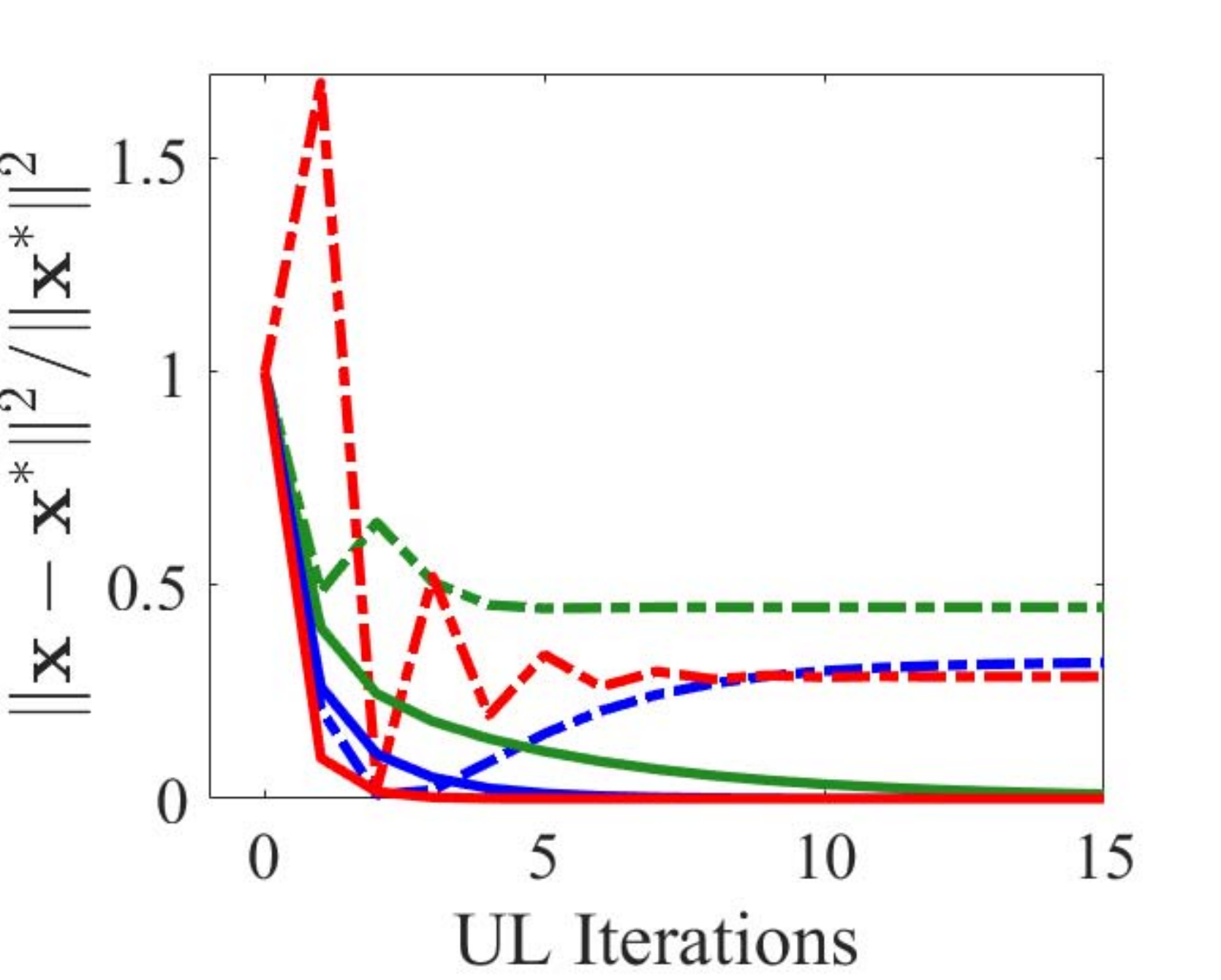}
		& \includegraphics[width=0.225\textwidth]{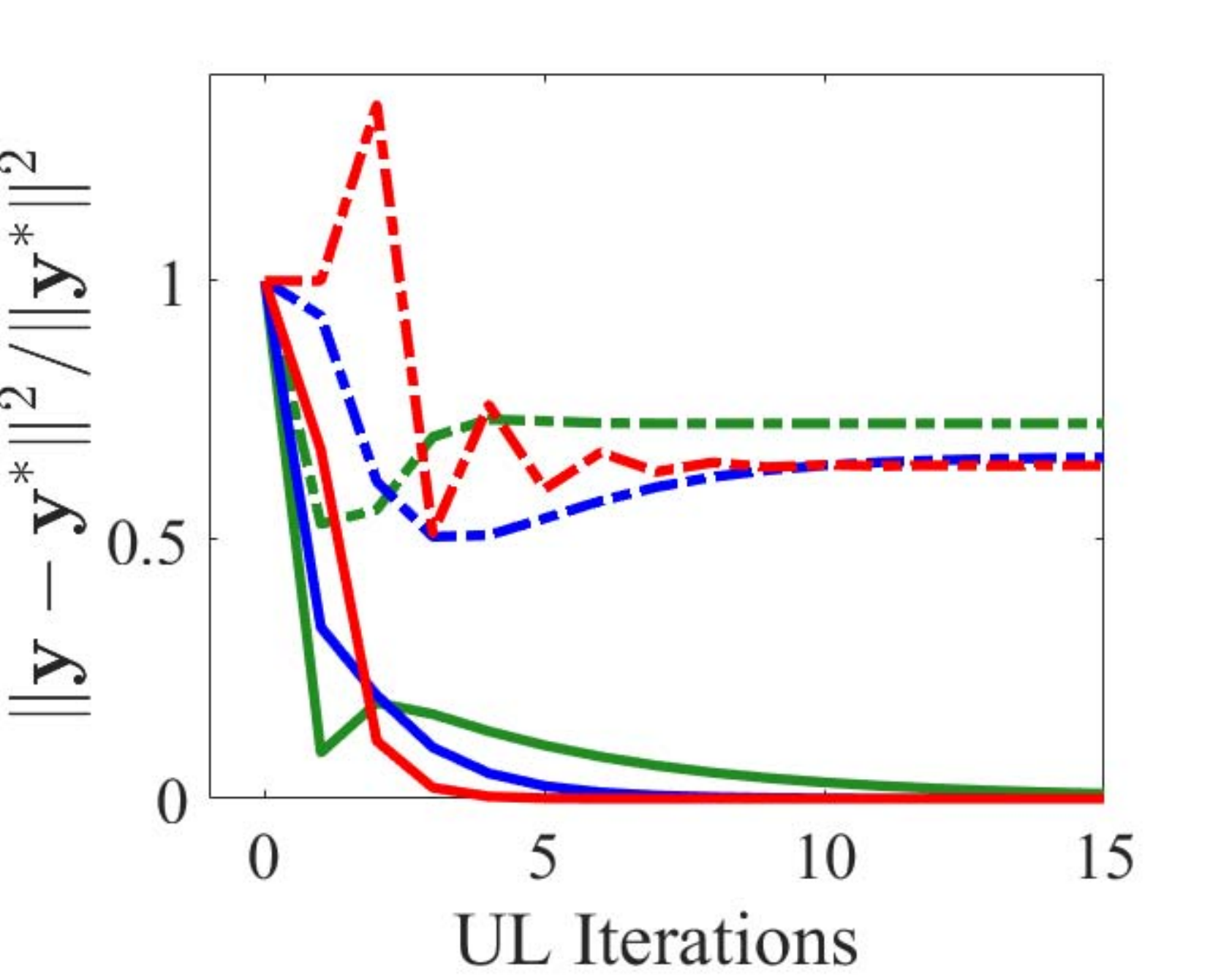}\\
	\end{tabular}
	\caption{Illustrating the numerical performance of first-order BLPs algorithms with different LL iterations (i.e., $K=8,16,64$). We fix initialization as $\mathbf{x}_0=0$ and $\mathbf{y}_0=(2,2)$. The dashed and solid curves denotes the results of RHG and BDA, respectively. The legend is only plotted in the first subfigure.} \label{fig:toy_inner}
\end{figure*}

\section{Experimental Results}\label{sec:exp}

In this section, we first verify the theoretical findings and then evaluate the performance of our proposed method on different problems, such as hyper-parameter optimization and meta learning. We conducted these experiments on a PC with Intel Core i7-7700 CPU (3.6 GHz), 32GB RAM and an NVIDIA GeForce RTX 2060 6GB GPU.

\subsection{Synthetic BLPs}

Our theoretical findings are investigated based on the synthetic BLPs described in Section~\ref{sec:ce}. As stated above, this deterministic bi-level formulation satisfies all the assumptions required in Section~\ref{sec:theory}, but it cannot meet the LLS condition considered in ~\cite{finn2017model,franceschi2017forward,franceschi2018bilevel,shaban2018truncated}. Here, we fix the learning rate parameters $s_u=0.7$ and $s_l=0.2$ in this experiment.

In Figure~\ref{fig:toy_initial}, we plotted numerical results of BDA and one of the most representative bi-level FOMs (i.e., Reverse Hyper-Gradient (RHG)~\cite{franceschi2017forward,franceschi2018bilevel}) with different initialization points. We considered different numerical metrics, such as $|F-F^*|$, $|f-f^*|$, ${\|\x-\x^*\|^2/\|\x^*\|^2}$, and ${\|\y-\y^*\|^2/\|\y^*\|^2}$, for evaluations. It can be observed that RHG is always hard to obtain correct solution, even start from different initialization points. This is mainly because that the solution set of the LL subproblem in Eq.~\eqref{eq:ce} is not a singleton, which does not satisfy the fundamental assumption of RHG. In contrast, our BDA aggregated the UL and LL information to perform the LL updating, thus we are able to obtain true optimal solution in all these scenarios. The initialization actually only slightly affected on the convergence speed of our iterative sequences.

Figure~\ref{fig:toy_inner} further plotted the convergence behaviors of BDA and RHG with different LL iterations (i.e., $K$). We observed that the results of RHG cannot be improved by increasing $K$. But for BDA, the three iterative sequences (with $K=8,16,64$) are always converged and the numerical performance can be improved by performing relatively more LL iterations. In the above two figures, we set $\alpha_k=0.5/k$, $k= 1,\cdots, K$. 

Figure~\ref{fig:toy_alpha} evaluated the convergence behaviors of BDA with different choices of $\alpha_k$. By setting $\alpha_k=0$, we was unable to use the UL information to guide the LL updating, thus it is hard to obtain proper feasible solutions for the UL subproblem. When choosing a fixed $\alpha_k$ in $(0,1)$ (e.g., $\alpha_k=0.5$), the numerical performance can be improved but the convergence speed was still slow. Fortunately, we followed our theoretical findings and introduced an adaptive strategy to incorporate UL information into LL iterations, leading to nice convergence behaviors for both UL and LL variables.

\subsection{Hyper-parameter Optimization} 

Hyper-parameter optimization aims choosing a set of optimal hyper-parameters for a given machine learning task. In this experiment, we consider a specific hyper-parameter optimization example, known as data hyper-cleaning~\cite{franceschi2017forward,shaban2018truncated}, to evaluate our proposed bi-level algorithm. In this task, we need to train a linear classifier on a given image set, but part of the training labels are corrupted. Following~\cite{franceschi2017forward,shaban2018truncated}, here we consider softmax regression (with parameters $\y$) as our classifier and introduce hyper-parameters $\x$ to weight samples for training.

Specifically, let $\ell(\y;\u_i,\v_i)$ be the cross-entropy function with the classification parameter $\y$ and data pairs $(\u_i,\v_i)$ and denote $\mathcal{D}_{\mathtt{tr}}$ and $\mathcal{D}_{\mathtt{val}}$ as the training and validation sets, respectively. Then we can define the LL objective as the following weighted training loss:  $$f(\x,\y)=\sum_{(\u_i,\v_i)\in\mathcal{D}_{\mathtt{tr}}}[\sigma(\x)]_i\ell(\y;\u_i,\v_i),$$ where $\x$ is the hyper-parameter vector to penalize the objective for different training samples. Here $\sigma(\x)$ denotes the element-wise sigmoid function on $\x$ and is used to constrain the weights in the range $[0,1]$. For the UL subproblem, we define the objective as the cross-entropy loss with $\ell_2$ regularization on the validation set, i.e., $$F(\x,\y)=\sum_{(\u_i,\v_i)\in\mathcal{D}_{\mathtt{val}}}\ell(\y(\x);\u_i,\v_i)+\lambda\|\y(\x)\|^2,$$ where $\lambda>0$ is the trade-off parameter and  fixed as $10^{-4}$.

\begin{figure}[h]
	\centering \begin{tabular}{c@{\extracolsep{0.2em}}c}
		\includegraphics[width=0.23\textwidth]{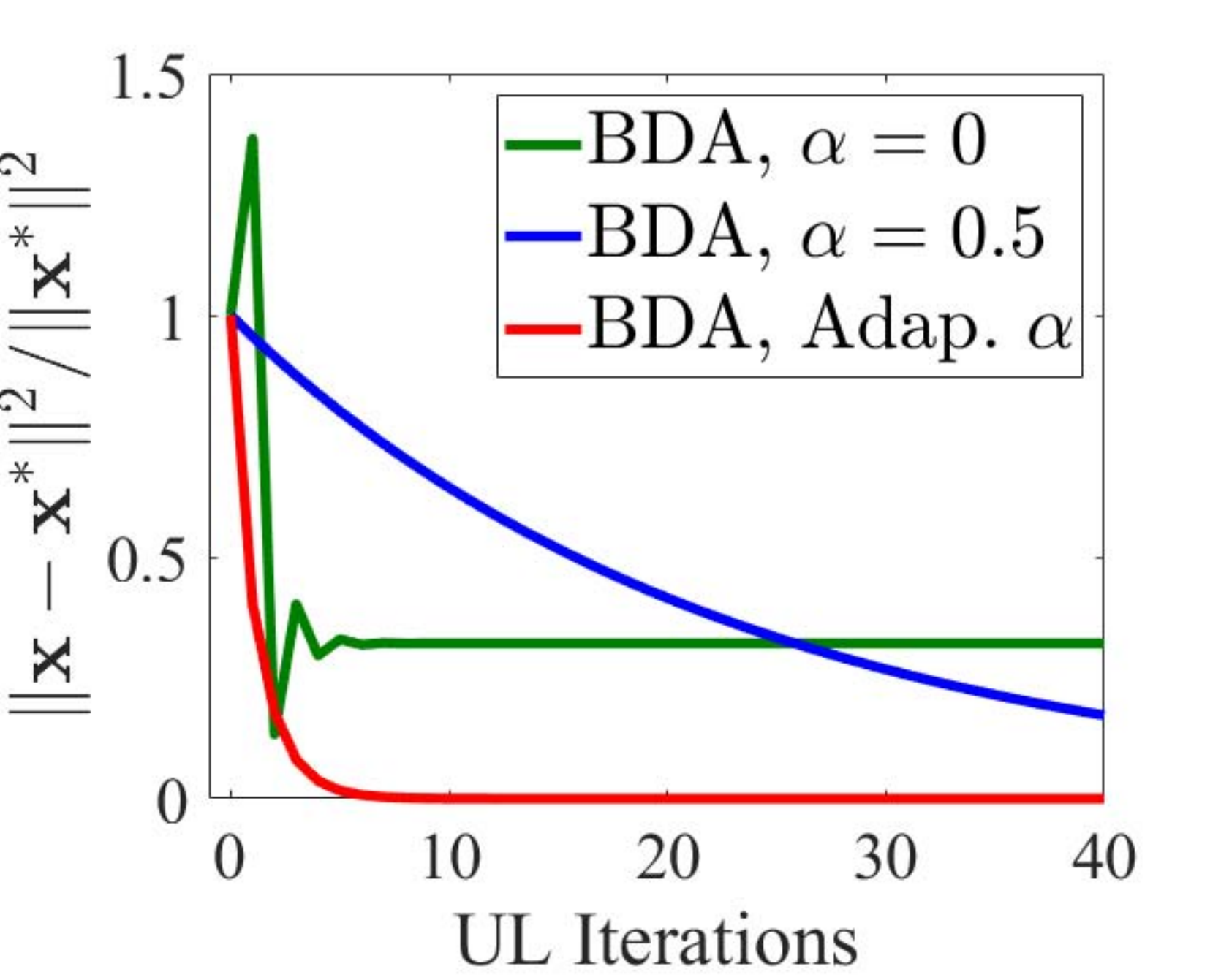}
		& \includegraphics[width=0.23\textwidth]{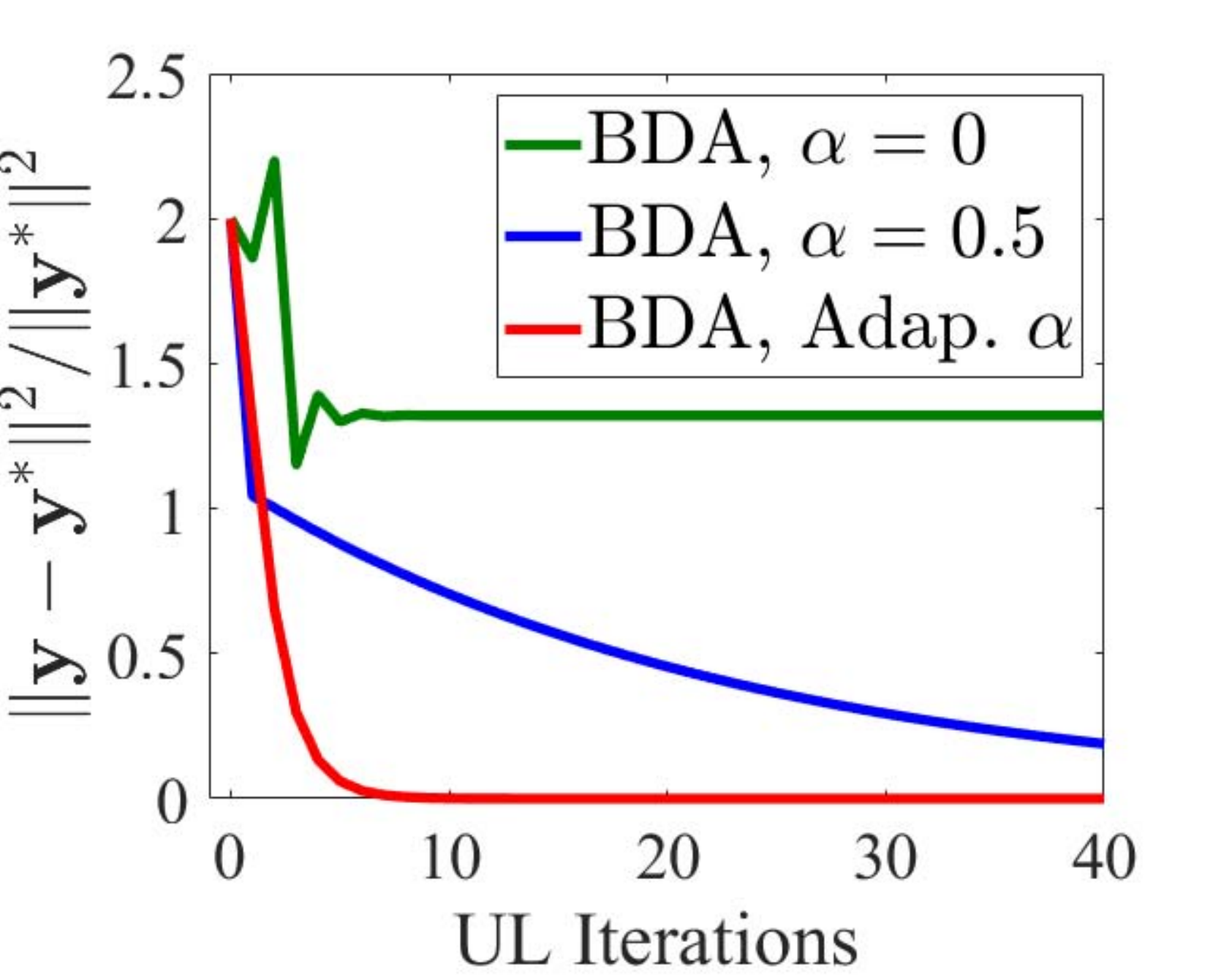}\\
	\end{tabular}
	\caption{Illustrating the numerical performance of BDA with fixed $\alpha_k$ (e.g., $\alpha_k=0, 0.5$) and adaptive $\alpha_k$ (e.g., $\{\alpha_k=0.9/k\}$, denoted as ``Adap. $\alpha$''). The initialization and LL iterations are fixed as $\mathbf{x}_0=0$, $\mathbf{y}_0=(2,2)$, and $K=16$, respectively. } \label{fig:toy_alpha}
\end{figure}

\begin{table}[t]
	\caption{Data hyper-cleaning accuracy of the compared methods with different number of LL iterations (i.e., $K=50, 100, 200, 400, 800$) on MNIST.}
	\label{tab:HyperClearning_1}
	\centering
	\vskip 0.12in
	\begin{tabular}{ | c | c | c | c | c | c |}
		\hline
		\multicolumn{1}{|c|}{\multirow{2}*{Method}}
		& \multicolumn{5}{c|}{No. of LL Iterations ($K$)}\\
		\cline{2-6}
		& 50 & 100 & 200 & 400 & 800\\
		\hline\hline
		RHG  & 88.96 &  89.73 & 90.13 &  90.19 & 90.15  \\
		\hline
		T-RHG  & 87.90 & 88.28 & 88.50 & 88.52 & 89.99 \\
		\hline		
		BDA & \textbf{89.12} & \textbf{90.12} & \textbf{90.57} & \textbf{90.81}& \textbf{90.86}\\
		\hline
	\end{tabular}
\end{table}

We applied our BDA together with the baselines, RHG and Truncated RHG (T-RHG)~\cite{shaban2018truncated}, to solve the above BLPs problem on MNIST database~\cite{lecun1998gradient}. Both the training and the validation sets consist of 7000 class-balanced samples and the remaining 56000 samples are used as the test set. We adopted the architectures used in RHG as the feature extractor for all the compared methods. For T-RHG, we chose $25$-step truncated back-propagation to guarantee its convergence. Table~\ref{tab:HyperClearning_1} reported the averaged accuracy for all these compared methods with different number of LL iterations (i.e., $K=50, 100, 200, 400, 800$). We observed that RHG outperformed T-RHG. While BAD consistently achieved the highest accuracy. Our theoretical results suggested that most of the improvements in BDA should come from the aggregations of the UL and LL information. The results also showed that more LL iterations are able to improve the final performances in most cases.

\subsection{Meta Learning}

The aim of meta learning is to learn an algorithm that should work well on novel tasks. In particular, we consider the few-shot learning problem \cite{vinyals2016matching,qiao2018few}, where each task is a $N$-way classification and it is to learn the hyper-parameter $\x$  such that each task can be solved only with $M$ training samples (i.e., $N$-way $M$-shot). 

Following the experimental protocol used in recent works, we separate the network architecture into two parts: the cross-task intermediate representation layers (parameterized by $\x$) outputs the meta features and the multinomial logistic regression layer (parameterized by $\y^j$) as our ground classifier for the $j$-th task. We also collect a meta training data set $\mathcal{D}=\{\mathcal{D}^j\}$, where $\mathcal{D}^j=\mathcal{D}_{\mathtt{tr}}^j\cup\mathcal{D}_{\mathtt{val}}^j$ is linked to the $j$-th task. Then for the $j$-th task, we consider the cross-entropy function $\ell(\x,\y^j;\mathcal{D}_{\mathtt{tr}}^j)$ as the task-specific loss and thus the LL objective can be defined as 
\begin{equation*}
\begin{array}{c}
f(\x,\{\y^j\})=\sum\limits_{j}\ell(\x,\y^j;\mathcal{D}_{\mathtt{tr}}^j).
\end{array}
\end{equation*}
As for the UL objective, we also utilize cross-entropy function but define it based on $\{\mathcal{D}_{\mathtt{val}}^j\}$ as 
\begin{equation*}
\begin{array}{c}
F(\x,\{\y^j\})=\sum\limits_{j}\ell(\x,\y^j;\mathcal{D}_{\mathtt{val}}^j).
\end{array}
\end{equation*}

Our experiments are conducted on two widely used benchmarks, i.e., Ominglot~\cite{lake2015human}, which contains 1623 hand written characters from 50 alphabets and MiniImageNet~\cite{vinyals2016matching}, which is a subset of ImageNet~\cite{deng2009imagenet} and includes 60000 downsampled images from 100 different classes. We followed the experimental protocol used in MAML~\cite{finn2017model} and compared our BDA to several state-of-the-art approaches, such as MAML~\cite{finn2017model}, Meta-SGD~\cite{li2017meta}, Reptile~\cite{nichol2018first}, RHG, and T-RHG.
\begin{table}[t]
	\caption{The averaged few-shot classification accuracy on Omniglot ($N=5,20$ and $M=1,5$). }
	\label{tab:omniglot}
	\centering 	
	\vskip 0.12in
	\begin{tabular}{|p{1.6cm}<{\centering}|p{1.2cm}<{\centering}|p{1.0cm}<{\centering}|p{1.2cm}<{\centering}|p{1.0cm}<{\centering}|}
		\hline
		\multirow{2}*{Method}
		&\multicolumn{2}{c|}{ $5$-way}&\multicolumn{2}{c|}{$20$-way}\\
		\cline{2-5}
		& $1$-shot & $5$-shot & $1$-shot & $5$-shot\\
		\hline\hline
		MAML &98.70&\textbf{99.91}&95.80&98.90 \\
		\hline
		Meta-SGD&97.97&98.96&93.98&98.40\\
		\hline
		Reptile &97.68&99.48 &89.43&97.12\\
		\hline\hline
		RHG&98.60 &99.50&95.50&98.40\\
		\hline
		T-RHG  & 98.74 & 99.52 & 95.82 & 98.95\\
		\hline
		BDA&\textbf{99.04}&99.62&\textbf{96.50}&\textbf{99.10}\\
		\hline
	\end{tabular}
\end{table}
\begin{table}[h]
	\caption{The few-shot classification performances on MiniImageNet ($N=5$ and $M=1$). The second column reported the averaged accuracy after converged. The rightmost two columns compared the UL Iterations (denoted as ``UL Iter.''), when achieving almost the same accuracy ($\approx 44\%$). Here  ``Ave. $\pm$ Var. (Acc.)'' denotes the averaged accuracy and the corresponding variance.}
	\label{tab:mini}
	\centering 
	\vskip 0.12in
	\begin{tabular}{|c | c ||c| c |}
		\hline
		Method & Acc. & Ave. $\pm$ Var. (Acc.)  & UL Iter. \\
		\hline\hline
		RHG&$48.89$  & 44.46 $\pm$ 0.78 & 3300\\
		\hline
		T-RHG & $47.67$ & 44.21 $\pm$ 0.78 & 3700\\
		\hline
		BDA & \textbf{49.08} & 44.24 $\pm$ 0.79 & \textbf{2500} \\
		\hline
	\end{tabular}
\end{table}
\begin{figure}[t]
	\centering \begin{tabular}{c@{\extracolsep{0.2em}}c}
		\includegraphics[width=0.23\textwidth]{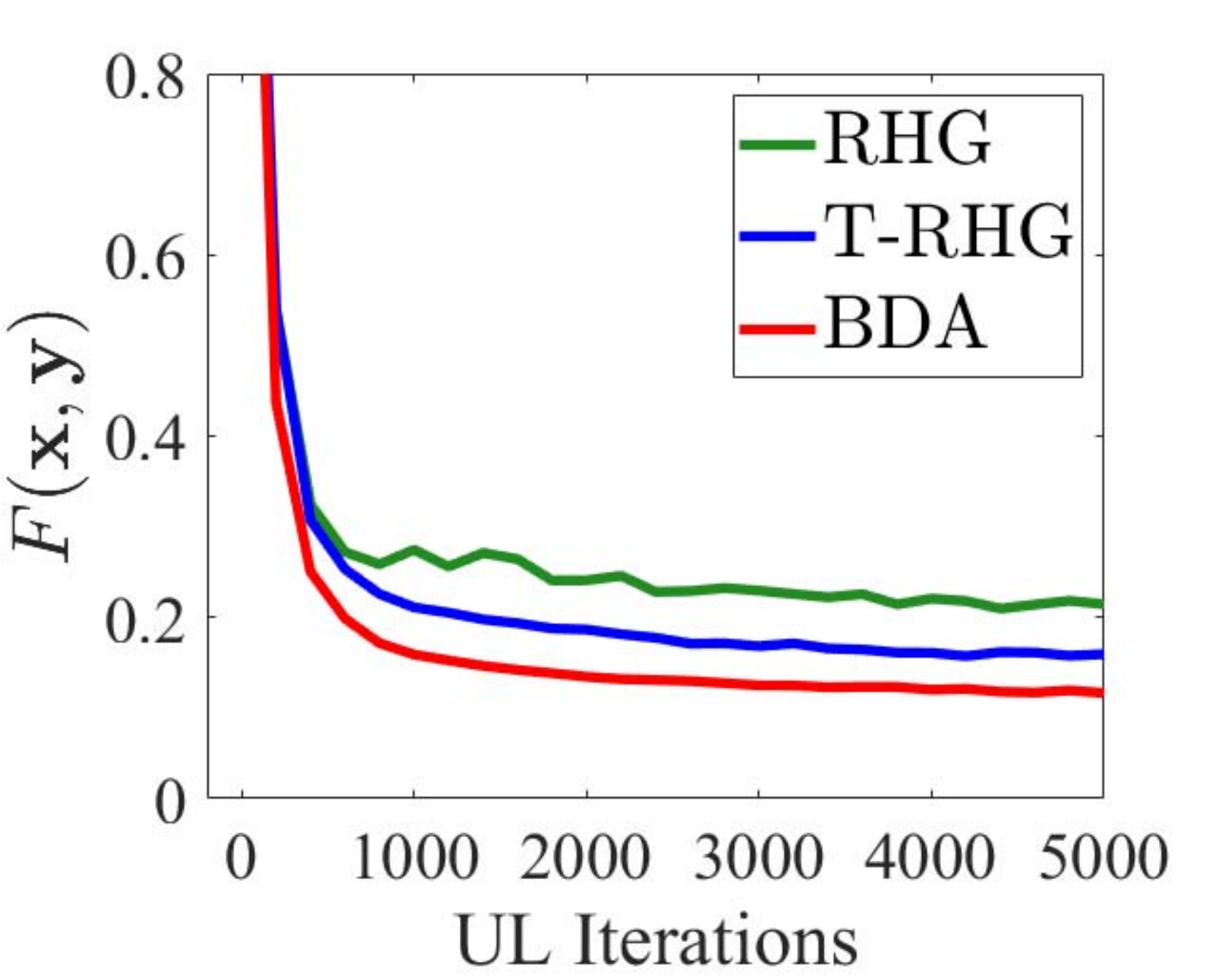}
		&\includegraphics[width=0.23\textwidth]{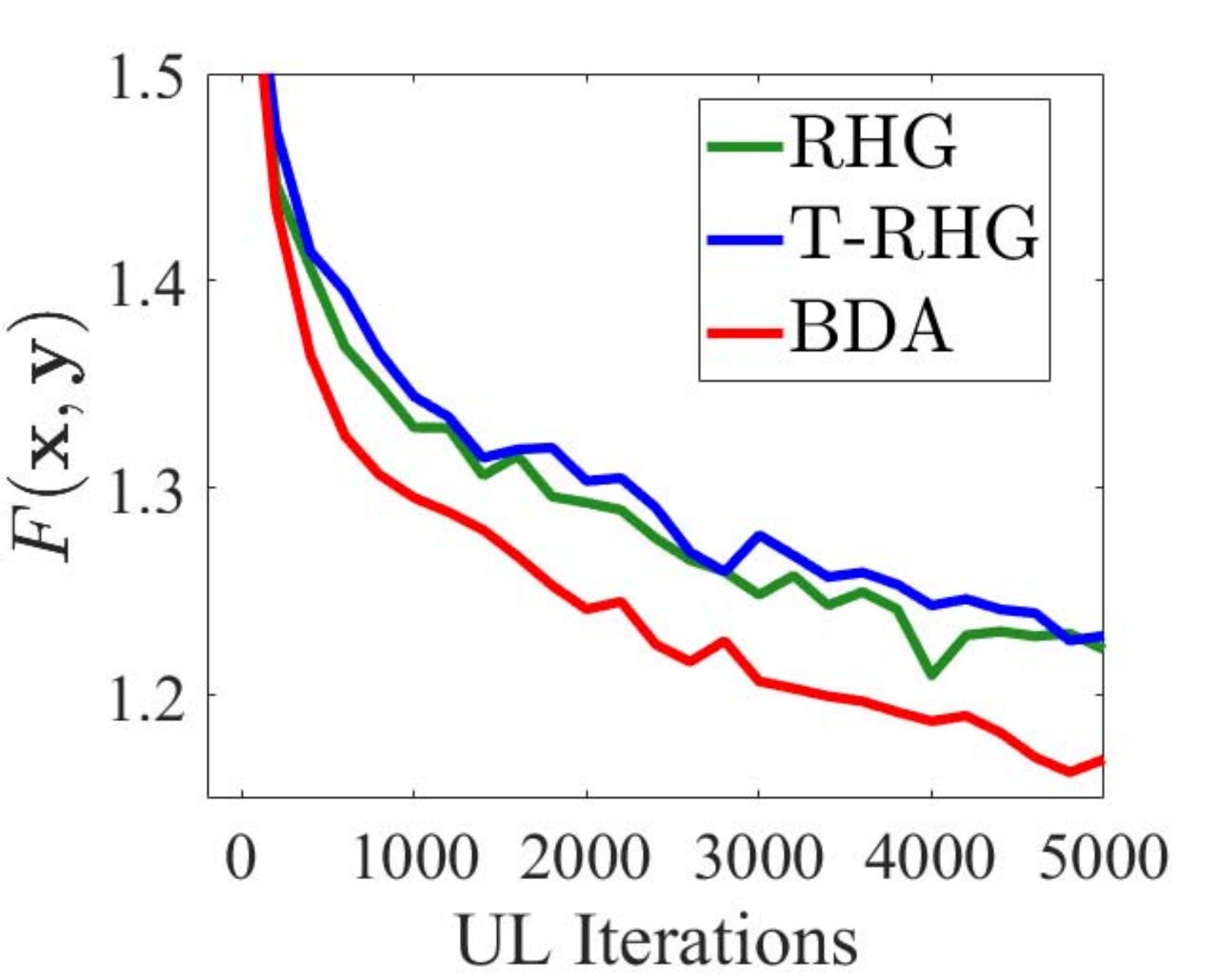}\\
	\end{tabular}
	\caption{Illustrating the validation loss (i.e., UL objectives $F(\mathbf{x},\mathbf{y})$) for three BLPs based methods on few-shot classification task. The curves in left and right subfigures are based on 5-way 1-shot results in Tables~\ref{tab:omniglot} and \ref{tab:mini}, respectively.}
	\label{fig:few-shot}
\end{figure}

It can be seen in Table~\ref{tab:omniglot} that BDA compared well to these methods and achieved the highest classification accuracy except in the 5-way 5-shot task. In this case, practical performance of BDA was slightly worse than MAML. We further conducted experiments on the more challenging MiniImageNet data set. In the second column of Table~\ref{tab:mini}, we reported the averaged accuracy of three first-order BLPs based methods (i.e., RHG, T-RHG and BDA). Again, the performance of BDA is better than RHG and T-RHG. In the rightmost two columns, we also compared the number of averaged UL iterations when they achieved almost the same accuracy ($\approx 44\%$). These results showed that BDA needed the fewest iterations to achieve such accuracy. The validation loss on Omnglot and MiniImageNet about 5-way 1-shot are plotted in Figure.~\ref{fig:few-shot}

\section{Conclusions}
	
The proposed BDA is a generic first-order algorithmic scheme to address BLPs. We first designed a counter-example to indicate that the existing bi-level FOMs in the absence of the LLS condition may lead to incorrect solutions. Considering BLPs from the optimistic bi-level viewpoint, BDA could reformulate the original models in Eqs.~\eqref{eq:blp}-\eqref{eq:lower-level} as the composition of a single-level subproblem (w.r.t. $\x$) and a simple bi-level subproblem (w.r.t., $\y$). We established a general proof recipe for bi-level FOMs and proved the convergence of BDA without the LLS assumption. As a nontrivial byproduct, we further improved convergence results for those existing schemes. Extensive evaluations showed the superiority of BDA for different applications.

\section*{Acknowledgements}

This work was supported by the National Natural Science Foundation of China (Nos. 61922019, 61672125, 61733002, 61772105 and 11971220), LiaoNing Revitalization Talents Program (XLYC1807088), the Fundamental Research Funds for the Central Universities and the Natural Science Foundation of Guangdong Province 2019A1515011152. This work was also supported by the General Research Fund 12302318 from Hong Kong Research Grants Council. 
	

\nocite{langley00}

\bibliography{reference}
\bibliographystyle{icml2020}

\appendix

\section*{Appendix}

This supplemental material is organized as follows. In Section~\ref{sec:proofThm1}, we present the detailed proof of subsection 4.1. Section~\ref{sec:profTheorem2} and Section~\ref{sec:profTheorem3} provide detailed proofs of subsection 4.2 and subsection 4.3 respectively. Section~\ref{sec:nonsmooth} then proves some extended theoretical results, including the local convergence behaviors of BDA, the algorithmic scheme and convergence properties of BDA for BLPs with nonsmooth LL objective.

\section{Proof of Section 4.1 }\label{sec:proofThm1}



\subsection{Proof of Theorem 1}

\begin{proof}
	Since $\X$ is compact, we can assume without loss of generality that $\x_K \rightarrow \bar{\x} \in \X$. For any $\epsilon > 0$, there exists $k(\epsilon) > 0$ such that whenever $K > k(\epsilon) $, we have
	\begin{equation*}
	\sup_{\x \in \X} \mathtt{dist}(\y_{K}(\x),\S(\x)) \le \frac{\epsilon}{2L_0}.
	\end{equation*}
	Thus, for any $\x \in \X$, there exists $\y^*(\x) \in \S(\x)$ such that
	\begin{equation*}
	\|\y_{K}(\x) - \y^*(\x)\| \le \frac{\epsilon}{L_0}.
	\end{equation*}
	Therefore, for any $\x \in \X$, we have 
	\begin{equation*}
	\begin{aligned}
	\varphi(\x) &= \inf_{\y \in \S(\x) } F(\x, \y) \\
	&\le F(\x, \y^*(\x)) \\
	&\le F(\x, \y_{K}(\x)) + L_0\|\y_{K}(\x) - \y^*(\x)\|\\
	&\le \varphi_K(\x) + \epsilon.
	\end{aligned}
	\end{equation*}
	This implies that,
	for any $\epsilon > 0$, there exists $k(\epsilon) > 0$ such that whenever $K> k(\epsilon) $,it holds
	\begin{equation*}
	\varphi(\x_{K}) \le \varphi_K(\x_{K}) + \epsilon \le \varphi_K(\x) + \epsilon,\quad \forall \x \in \X.
	\end{equation*}
	Taking $K \rightarrow \infty$ and by the LSC of $\varphi$, we have
	\begin{equation*}
	\begin{aligned}
	\varphi(\bar{\x}) &\le \liminf_{K \rightarrow \infty}\varphi(\x_{K})\\
	& \le \liminf_{K\to\infty}\varphi_K(\x_{K}) + \epsilon \\
	& \le \lim_{K \rightarrow \infty}\varphi_K(\x) + \epsilon = \varphi(\x) + \epsilon, \quad \forall \x \in \X.
	\end{aligned}
	\end{equation*}
	By taking $\epsilon \rightarrow 0$, we have
	\begin{equation*}
	\varphi(\bar{\x}) \le \varphi(\x), \quad \forall \x \in \X,
	\end{equation*}
	which implies $\bar{\x} \in \arg\min_{\x \in \X} \varphi(\x)$. 
	
	We next show that $\inf_{\x \in \X}\varphi_K(\x) \rightarrow \inf_{\x \in \X} \varphi(\x)$ as $K \rightarrow \infty$. If this is not true, then there exist $\delta > 0$ and sequence $\{l \} \subseteq \mathbb{N}$ such that
	\begin{equation}\label{eq:contra_dist}
	\left|\inf_{\x \in \X}\varphi_{l}(\x) - \inf_{\x \in \X} \varphi(\x) \right| > \delta, \quad \forall l.
	\end{equation}
	For each $l$, there exists $\x_{l} \in \X$ such that $$\varphi_{l}(\x_{l}) \le \inf_{\x \in \X} \varphi_{l}(\x) + \delta/2.$$ Since $\X$ is compact, we can assume without loss of generality that $\x_{l} \rightarrow \tilde{\x} \in \X$. For any $\epsilon > 0$, there exists $k(\epsilon) > 0$ such that whenever $l > k(\epsilon) $, the following holds
	\begin{equation*}
	\begin{aligned}
	\varphi(\x_{l}) &\le \varphi_{l}(\x_{l}) + \epsilon \\
	&\le \inf_{\x \in \X} \varphi_{l}(\x) + \delta/2 + \epsilon \\
	&\le \varphi_{l}(\x) + \delta/2 + \epsilon, \quad \forall \x \in \X.
	\end{aligned}
	\end{equation*} 
	By taking $l \rightarrow \infty$ and with the LSC of $\varphi$, we have
	\begin{equation*}
	\begin{aligned}
	\varphi(\tilde{\x}) &\le \liminf_{{l} \rightarrow \infty}\varphi(\x_{l}) \\
	&\le \liminf_{l \rightarrow \infty}\left(\inf_{\x \in \X} \varphi_{l}(\x)\right) + \delta/2 + \epsilon \\
	&\le \limsup_{l \rightarrow \infty}\left(\inf_{\x \in \X} \varphi_{l}(\x)\right) + \delta/2 + \epsilon \\
	&\le \varphi(\x)+ \delta/2  + \epsilon, \quad \forall \x \in \X.
	\end{aligned}
	\end{equation*}
	Then, by taking $\epsilon \rightarrow 0$, we have
	\begin{equation*}
	\begin{aligned}
	\inf_{\x \in \X} \varphi(\x) &\le \liminf_{l \rightarrow \infty}\left(\inf_{\x \in \X} \varphi_{l}(\x)\right) + \delta/2 \\
	&\le \limsup_{l \rightarrow \infty}\left(\inf_{\x \in \X} \varphi_{l}(\x)\right) + \delta/2 \\
	& \le \inf_{\x \in \X} \varphi(\x)+ \delta/2,
	\end{aligned}
	\end{equation*}
	which implies a contradiction to Eq.~\eqref{eq:contra_dist}. Thus we have $\inf_{\x \in \X}\varphi_K(\x) \rightarrow \inf_{\x \in \X} \varphi(\x)$ as $K \rightarrow \infty$.
\end{proof}

\section{Proofs of Section 4.2}\label{sec:profTheorem2}


\subsection{Proof of Lemma \ref{lemma:tildeS_bounded}}

\begin{proof}
	We prove this result by providing a contradiction, that is, we have $\{\x^t\} \subseteq \X$ and $\y^t \in \tilde{\S}(\x^t)$ such that $\|\y^t\| \rightarrow +\infty$. As $\X$ is compact, we can assume without loss of generality that $\x^t \rightarrow \bar{\x} \in \X$. Since $F(\x,\y)$ is level-bounded in $\y$ locally uniformly in $\x\in\X$, we must have $\varphi(\x^t)  = F(\x^t,\y^t) \rightarrow + \infty$. On the other hand, for any $\epsilon > 0$, let $\bar{\y} \in \S(\bar{\x})$ satisfy $F(\bar{\x},\bar{\y})\le \varphi(\bar{\x}) + \epsilon$. As $F$ is continuous at $(\bar{\x},\bar{\y})$, there exists $\delta_0 > 0$ such that 
	\begin{equation*}
	F(\x,\y) \le F(\bar{\x},\bar{\y}) + \epsilon,~\forall (\x,\y) \in \mathbb{B}_{\delta_0}(\bar{\x},\bar{\y}).
	\end{equation*} 
	As $\S(\x)$ is ISC at $\bar{\x}$ relative to $\X$, then it follows that there exists $\frac{\sqrt{2}}{2}\delta_0 \ge \delta > 0$ satisfying $$\S(\x) \cap \mathbb{B}_{\frac{\sqrt{2}}{2}\delta_0}(\bar{\y}) \neq \varnothing,\ \forall \x \in \mathbb{B}_{\delta}(\bar{\x}) \cap \X.$$ Therefore, for any $\x \in \mathbb{B}_{\delta}(\bar{\x}) \cap \X$, there exists $\y \in \S(\x)$ satisfying $(\x,\y) \in \mathbb{B}_{\delta_0}(\bar{\x},\bar{\y})$ and thus $F(\x,\y) \le F(\bar{\x},\bar{\y}) + \epsilon$. Consequently, for any $\x \in \mathbb{B}_{\delta}(\bar{\x}) \cap \X$, we have 
	\begin{equation*}
	\varphi(\x) = \min_{\y \in \S(\x) } ~~ F(\x,\y) \le F(\bar{\x},\bar{\y}) + \epsilon = \varphi(\bar{\x})+ 2\epsilon,
	\end{equation*}
	which contradicts to $\varphi(\x^t) \rightarrow \infty$.
\end{proof}


\subsection{Proof of Lemma \ref{lemma:f_USC}}

\begin{proof}
	For any sequence $\{\x^t\} \subseteq \X$ satisfying $\x^t \rightarrow \bar{\x} \in \X$, given any $\epsilon > 0$, let $\bar{\y} \in \mathbb{R}^m$ satisfy $f(\bar{\x},\bar{\y}) \le f^*(\bar{\x})+ \epsilon$. As $f$ is continuous at $(\bar{\x},\bar{\y})$, there exists $T > 0$ such that 
	\begin{equation*}
	f^*(\x^t) \le f(\x^t,\bar{\y}) \le f(\bar{\x},\bar{\y}) + \epsilon \le f^*(\bar{\x})+ 2\epsilon,\quad \forall t > T,
	\end{equation*}
	and thus
	\begin{equation*}
	\limsup_{t \rightarrow \infty} f^*(\x^t) \le f^*(\bar{\x}) + 2\epsilon.
	\end{equation*}
	By taking $\epsilon \rightarrow 0$, we get $\limsup_{k \rightarrow \infty} f^*(\x^t) \le f^*(\bar{\x})$.
\end{proof}


We next prove the uniform convergence of $\{\tilde{\y}_K(\x)\}$ towards the solution set $\mathcal{S}(\mathbf{x})$ through the uniform convergence of $\{f(\x,\tilde{\y}_K(\x))\}$.

\subsection{Proof of Proposition \ref{prop:yK_conver}}

\begin{proof}
	We are going to prove this statement by a contradiction. We assume that there exist bounded set $\mathcal{Y} \subseteq \mathbb{R}^m$, $\epsilon > 0$, sequences $\{(\x^t,\y^t)\} \subseteq \X \times \mathcal{Y}$ and $\{\delta_k\}$ with $\delta_k \rightarrow 0$ satisfying 
	\begin{equation*}
	f(\x^t,\y^t) - f^*(\x^t) \le \delta_k \ \text{and}\ \mathtt{dist}(\y^t,\S(\x^t)) > \epsilon.
	\end{equation*}
	Without loss of generality, we can assume that $\x^t \rightarrow \bar{\x} \in \X$ and $\y^t \rightarrow \bar{\y} \in \mathbb{R}^m$ as $t\to\infty$. According to the continuity of $f$ and the USC of $f^*$ from Lemma~\ref{lemma:f_USC}, we have
	\begin{equation*}
	0 \le f(\bar{\x},\bar{\y}) - f^*(\bar{\x}) \le \liminf_{t \rightarrow \infty}f(\x^t,\y^t) - f^*(\x^t) \le 0,
	\end{equation*}
	which implies $\bar{\y} \in \S(\bar{\x})$. However, as  $\mathtt{dist}(\y^t,\S(\x^t)) > \epsilon$, following from the ISC of $\S(\x)$ at $\bar{\x}$ and Proposition 5.11 of~\cite{rockafellar2009variational}, we have
	\begin{equation*}
	\begin{aligned}
	\mathtt{dist}(\bar{\y},\S(\bar{\x})) &\ge \limsup_{t \rightarrow \infty} \mathtt{dist}(\bar{\y},\S(\x^t)) \\
	&= \limsup_{t \rightarrow \infty} \left(\mathtt{dist}(\y^t,\S(\x^t)) + \|\y^t - \bar{\y}\|\right) \\
	&\ge \liminf_{t \rightarrow \infty} \mathtt{dist}(\y^t,\S(\x^t)) \ge \epsilon,
	\end{aligned}
	\end{equation*}
	which contradicts to $\bar{\y} \in \S(\bar{\x})$.
\end{proof}



\subsection{Proof of Proposition \ref{prop:varphi_LSC}}

\begin{proof}
	We assume that there exists $\bar{\x}\in\X$ satisfying $\x^t\to\bar{\x}$ as $t\to\infty$, then the following 
	\begin{equation*}
	\liminf_{\x \rightarrow \bar{\x}} \varphi(\x) < \varphi(\bar{\x}),
	\end{equation*}
	holds. Next, there exist $\epsilon > 0$ and sequences $\x^t \rightarrow \bar{\x}\in\X$ and $\y^t \in \S(\x^t)$ satisfying
	\begin{equation*}
	F(\x^t,\y^t) \le \varphi(\x^t) + \epsilon < \varphi(\bar{\x}) - \epsilon.
	\end{equation*}
	Furthermore, 
	since $F(\x,\y)$ is level-bounded in $\y$ locally uniformly in $\x\in\X$, we have that $\{ \y^t \}$ is bounded. Take a subsequence $\{\y^\nu\}$ of $\{\y^t\}$ such that $\y^\nu \rightarrow \hat{\y}$ and it follows from the OSC of $\S$ that $\hat{\y} \in \S(\bar{\x})$. Then we have
	\begin{equation*}
	\begin{aligned}
	\varphi(\bar{\x}) \le F(\bar{\x},\hat{\y}) \le \limsup_{t \rightarrow \infty} F(\x^t,\y^t) &= \limsup_{t \rightarrow \infty} \varphi(\x^t) \\
	&\le \varphi(\bar{\x}) - \epsilon,
	\end{aligned}
	\end{equation*}
	which implies a contradiction. Thus $$\varphi(\bar{\x}) \le \liminf_{\x \rightarrow \bar{\x}} \varphi(\x)$$ and we get the conclusion.
\end{proof}



\subsection{Proof of Theorem \ref{thm:conver_BDA}}

\begin{proof}
	We first show that $F(\x,\y)$ is level-bounded in $\y$ locally uniformly in $\x\in\X$. For any $\bar{\x} \in \X$, let $\{\x^t\} \subseteq \X$ with $\x^t \rightarrow \bar{\x}$ and $\{\y^t\} \in \mathbb{R}^m$ with $\|\y^t\| \rightarrow +\infty$. Then, with Assumption~1 
	we have
	\begin{equation*}
	\begin{aligned}
	F(\x^t,\y^t) \ge & F(\x^t,\y^1) + \langle \nabla_\y F(\x^t,\y^1), \y^t - \y^1 \rangle\\ 
	& + \frac{\sigma}{2}\| \y^t - \y^1 \|^2.
	\end{aligned}
	\end{equation*}
	As $F(\x,\cdot) : \mathbb{R}^m \rightarrow \mathbb{R}$ is Lipschitz continuous with uniform constant $L_0$ for any $\x \in \X$, we have $\| \nabla_\y F(\x^t,\y^1)\| \le L_0$. Then, by the continuity of $F$, with $\x^t \rightarrow \bar{\x} \in \X$, and $\|\y^t\| \rightarrow +\infty$, we have $F(\x^t,\y^t) \rightarrow +\infty$. Thus $F(\x,\y)$ is level-bounded in $\y$ locally uniformly in $\x\in\X$. Then with Proposition~\ref{prop:varphi_LSC} and assumptions in Theorem~\ref{thm:conver_BDA}, we get the LSC property of $\varphi$ on $\X$. And according to Lemma~\ref{lemma:tildeS_bounded}, there exists $M > 0$ such that $C_{\y^*(\x)} \le M$ for any $\y^*(\x) \in \tilde{\S}(\x)$ and $\x \in \X$. 
	Following Proposition~\ref{prop:f_conver}, there exists $C > 0$ such that for any $\x \in \X$ we have
	\begin{align*}
	&\|\y_K(\x)\| \le C, \quad \forall K \ge 0,\\
	&\|\y_K(\x) - \tilde{\y}_K(\x)\| \le \frac{C}{K},
	\end{align*}
	and 
	\begin{equation*}
	f(\tilde{\y}_K(\x)) - f^*(\x) \le \frac{C}{K}, \quad \forall K \ge 0.
	\end{equation*}
	Next, according to Proposition~\ref{prop:yK_conver}, for any $\epsilon > 0$, there exists $k(\epsilon) > 0$ such that whenever $K > \max \{2C/\epsilon, k(\epsilon) \}$ we have
	\begin{equation*}
	\begin{aligned}
	&\quad\sup_{\x \in \X} \mathtt{dist}(\y_{K}(\x),\S(\x)) \\
	&\le \|\y_K(\x) - \tilde{\y}_K(\x)\| + \sup_{\x \in \X} \mathtt{dist}(\tilde{\y}_{K}(\x),\S(\x)) \le \epsilon.
	\end{aligned}
	\end{equation*}
	Then it follows from Proposition~\ref{prop:f_conver} that $\varphi_K(\x) \rightarrow \varphi(\x)$ when $K \rightarrow \infty$ for any $\x \in \X$. 
\end{proof}

\section{Proofs of Section 4.3}\label{sec:profTheorem3}


\subsection{Proof of Lemma \ref{prop:revisting_continuous_S}}

\begin{proof}
	First, according to Proposition 4.4 of~\cite{bonnans2013perturbation}, we know that if $f(\x,\y) : \mathbb{R}^n \times \mathbb{R}^m \rightarrow \mathbb{R}$ is continuous on $\X \times \mathbb{R}^m$, level-bounded in $\y$ locally uniformly in $\x\in\X$, then $f^*(\x)$ is continuous on $\X$, $\S(\x)$ is OSC on $\X$ and locally bounded at $\bar{\x}$. Thus, for any $\bar{\x} \in \X$, $f^*(\x)$ is locally bounded at $\bar{\x}$. 
	As $\S(\x)$ is a single-valued mapping on $\X$ and $\S(\x)$ is OSC at $\bar{\x} \in \X$ and locally bounded at $\bar{\x}$, Upon Proposition 5.20 of~\cite{rockafellar2009variational}, we conclude that $\S(\x)$ is ISC at $\bar{\x}$, and thus continuous at $\bar{\x}$. This completes the proof.
\end{proof}


\subsection{Proof of Theorem \ref{thm:revisting_conver}}
\begin{proof}
	First, we get the continuity of $\S(\x)$ on $\X$ from Lemma~\ref{prop:revisting_continuous_S}. Then, by Proposition~\ref{prop:varphi_LSC}, we obtain the LSC of $\varphi(\x)$ on $\X$. From Proposition~\ref{prop:yK_conver} and Lemma~\ref{prop:revisting_continuous_S}, we have that for any $\epsilon>0$, there exists $k(\epsilon)>0$ such that whenever $K>k(\epsilon)$,  
	\begin{equation*}
	\sup_{\mathbf{x}\in\X}\mathtt{dist}(\mathbf{y}_{K}(\mathbf{x}),\S(\mathbf{x}))\leq\epsilon.
	\end{equation*}
	As $\S(\x)$ is a single-valued mapping on $\X$, we have $\varphi_K(\x) \rightarrow \varphi(\x)$ for any $\x \in \X$ as $K \rightarrow \infty$. 
\end{proof}

In the following two propositions, we assume that $f(\x,\cdot):\mathbb{R}^m\to\mathbb{R}$ is $L_f$-smooth and convex, $s_l \leq 1/L_f$. 

\subsection{Proof of Proposition \ref{pg}}
\begin{proof}
	This proposition can be directly obtained from Theorem 10.21 and Theorem 10.23 of~\cite{beck2017first}. 
\end{proof}
Then in the following proposition we can immediately verify our required assumption on $\{f(\x,\y_K(\x))\}$ in the absence of the strong convexity property on the LL objective.

\subsection{Proof of Proposition \ref{prop:fK-conver}}

\begin{proof}
	By the same arguments given in proof of Lemma~\ref{prop:revisting_continuous_S}, we can show that $\S(\x)$ is locally bounded at each point on $\X$ under Assumption~2. As $\X$ is compact, thus $\cup_{\x\in \X} \S(\x)$ is bounded. Then the conclusion follows from Proposition~\ref{pg} directly.
\end{proof}

\section{Extended Theoretical Results}\label{sec:nonsmooth}

\subsection{Local Convergence Results}\label{sec:local}

In this part, we analyze the local convergence behaviors of BDA. 
In fact, even if $\x_K$ is a local minimum of $\varphi_{K}(\x)$ with uniform neighborhood modulus $\delta > 0$, we can still obtain similar convergence results as that in Theorem~\ref{thm:general}. Such properties are summarized in the following theorem.
\begin{thm}\label{thm:local}
	Suppose both the LL solution set and UL objective convergence properties (stated in Section \ref{subsec:proofThm1}) hold and let $\x_K$ be a local minimum of $\varphi_{K}(\x)$ with uniform neighborhood modulus $\delta > 0$. Then we have that any limit point $\bar{\x}$ of the sequence $\{\x_K\}$ is a local minimum of $\varphi$, i.e., there exists $\tilde{\delta} > 0$ such that
	\[
	\varphi(\bar{\x}) \le \varphi(\x), \quad \forall \x \in \mathbb{B}_{\tilde{\delta}} (\bar{\x})\cap \X.
	\]
\end{thm}

\begin{proof}
	Since $\X$ is compact, we can assume without loss of generality that $\x_K \rightarrow \bar{\x} \in \X$ and $\x_K \in \mathbb{B}_{\delta/2} (\bar{\x})$ by considering a subsequence of $\{\x_K\}$. For any $\epsilon > 0$, there exists $k(\epsilon) > 0$ such that whenever $K > k(\epsilon) $, we have
	\begin{equation*}
	\sup_{\x \in \X} \mathtt{dist}(\y_{K}(\x),\S(\x)) \le \frac{\epsilon}{2L_0}.
	\end{equation*}
	Thus, for any $\x \in \X$, there exists $\y^*(\x) \in \S(\x)$ such that
	\begin{equation*}
	\|\y_{K}(\x) - \y^*(\x)\| \le \frac{\epsilon}{L_0}.
	\end{equation*}
	Therefore, for any $\x \in \X$, we have 
	\begin{equation*}
	\begin{aligned}
	\varphi(\x) &= \inf_{\y \in \S(\x) } F(\x, \y) \\
	&\le F(\x, \y^*(\x)) \\
	&\le F(\x, \y_{K}(\x)) + L_0\|\y_{K}(\x) - \y^*(\x)\|\\
	&\le \varphi_K(\x) + \epsilon.
	\end{aligned}
	\end{equation*}
	This implies that, for any $\epsilon > 0$, there exists $k(\epsilon) > 0$ such that whenever $K> k(\epsilon) $, we have
	\begin{equation*}\label{lem0}
	\varphi(\x_{K}) \le \varphi_K(\x_{K}) + \epsilon \le \varphi_K(\x) + \epsilon , \quad \forall \x \in \X.
	\end{equation*}
	Next, as $\x_K$ is a local minimum of $\varphi_{K}(\x)$ with uniform neighborhood modulus $\delta$, it follows
	\[
	\varphi_K(\x_K) \le \varphi_K(\x), \ \forall \x \in \mathbb{B}_{\delta} (\x_K)\cap \X.
	\]
	Since $\mathbb{B}_{\delta/2} (\bar{\x}) \subseteq \mathbb{B}_{\delta/2+\|\x_k - \bar{\x}\|}(\x_K) \subseteq \mathbb{B}_{\delta} (\x_K)$, we have that for any $\epsilon > 0$, $\forall \x \in \mathbb{B}_{\delta/2} (\bar{\x}) \cap \X,$ there exists $k(\epsilon) > 0$ such that whenever $K > k(\epsilon) $, 
	\[
	\begin{aligned}
	\varphi(\x_K) \le \varphi_K(\x_K) + \epsilon \le \varphi_K(\x) + \epsilon.
	\end{aligned}
	\]
	Taking $K \rightarrow \infty$ and by the LSC of $\varphi$, $\forall \x \in \mathbb{B}_{\delta/2} (\bar{\x}) \cap \X$, we have
	\begin{equation*}
	\begin{aligned}
	\varphi(\bar{\x}) &\le \liminf_{K \rightarrow \infty}\varphi(\x_K)\\
	& \le \liminf_{K \rightarrow \infty}\varphi_K(\x_K) + \epsilon \\
	& \le \lim_{K \rightarrow \infty}\varphi_K(\x) + \epsilon = \varphi(\x) + \epsilon.
	\end{aligned}
	\end{equation*}
	By taking $\epsilon \rightarrow 0$, we have
	\begin{equation*}
	\varphi(\bar{\x}) \le \varphi(\x), \quad \forall \x \in \mathbb{B}_{\delta/2} (\bar{\x})\cap \X,
	\end{equation*}
	which implies $\bar{\x} \in \arg\min_{\x \in \mathbb{B}_{\delta/2} (\bar{\x})\cap \X} \varphi(\x)$, i.e, $\bar{x}$ is a local minimum of $\varphi$.
\end{proof}

\subsection{Nonsmooth LL Objective}

It is well-known that a variety of nonsmooth regularization techniques (e.g., $\ell_1$-norm regularization) have been utilized in learning and vision areas. So in this section, we briefly discuss a potential extension of BDA for BLPs with the nonsmooth LL objective, e.g.,
\begin{equation}\label{eq:llobj}
\S(\x)=\arg\min\limits_{\y}h(\x,\y)=f(\x,\y) + g(\x, \y).
\end{equation}
Here we consider $f$ as a function with the same properties as that in our above analysis, while $g$ is convex but not necessarily smooth, w.r.t. $\y$ and continuous w.r.t. $(\x,\y)$. Since $g$ is not necessarily differentiable w.r.t. $\y$, these existing gradient-based first-order BLP methods are not available for this problem. Fortunately, we demonstrate that by slightly modifying our inner updating rule, BDA can be directly extended to address BLPs with the nonsmooth LL objective in Eq.~\eqref{eq:llobj}. Specifically, we first write the descent direction of the LL subproblem as 
\begin{equation*}
{\mathbf{d}}^{{h}}_k(\mathbf{x})=\mathbf{y}_k-\mathtt{prox}_{s_lg(\x,\cdot)}( \y_k-s_l\nabla_{\mathbf{y}} f(\mathbf{x},\mathbf{y}_{k})),
\end{equation*}
where $\mathtt{prox}_{s_lg(\x,\cdot)}$ denotes the proximal operator w.r.t. the nonsmooth function $g(\x,\cdot)$ and step size $s_l$. Then by aggregating $\mathbf{d}^{F}_k(\mathbf{x})$ and $\mathbf{d}^{h}_k(\mathbf{x})$, we derive a new $\T_k$ to handle BLPs with the nonsmooth composite LL objective $h$, i.e., 
\begin{equation}
\T_{k+1}(\x,\y_k(\x))=\y_k -\left( \alpha_k\mathbf{d}^{{h}}_k(\mathbf{x})+(1-\alpha_k)\mathbf{d}^{{F}}_k(\mathbf{x})\right),\label{eq:lower-non}
\end{equation}
where $\alpha_k\in(0,1]$. In fact, since explicitly estimating the subgradient information of some proximal operators may be computationally infeasible in practice, one may apply automatic differentiation through the dynamical system with approximation techniques~\cite{wang2016proximal,rajeswaran2019meta} to obtain $\frac{d\varphi_{K}}{d\x}$, where $\varphi_{K}(\mathbf{x})=F(\mathbf{x},\mathbf{y}_K(\mathbf{x}))$.

We are now in the position to extend the converge properties of BDA for BLPs in Eq.~\eqref{eq:lower-non} from smooth LL case to nonsmooth LL case. Similar to the discussion in the smooth case, our analysis could follow the following roadmap:

Step 1: Denoting $\tilde{\S}(\x) = \arg\min_{\y \in \S(\x) } F(\x,\y)$ and further $h^{\ast}(\x)=\min_{\y}h(\x,\y)$, as extensions to Lemma~\ref{lemma:tildeS_bounded} and Lemma~\ref{lemma:f_USC}, we shall derive the boundedness of $\tilde{\S}(\x)$ and the USC of $h^{\ast}(\x)$ for the nonsmooth LL case, respectively. The proofs are indeed straightforward and purely technical, thus omitted here. 

Step 2: As an extension to Proposition~\ref{prop:f_conver} which focuses on the smooth case, we may derive the following convergence results regarding $\{\y_K(\x)\}$ in the light of the general fact stated in~\cite{sabach2017first}.
\begin{prop}\label{prop:f+g_conver}
	Suppose Assumption 1 is satisfied, $g$ is continuous and convex w.r.t. $\y$, and let $\{\y_K\}$ be defined as in Eq.~\eqref{eq:lower-non}, $s_l \in (0,1/L_f]$, $s_u \in (0,2/(L_F+\sigma)]$, $$\alpha_k = \min \left\{2\gamma/k(1-\beta),1-\varepsilon \right\},$$ with $k \ge 1$, $\varepsilon>0$, $\gamma \in (0,1]$ and $$\beta = \sqrt{1-2s_u\sigma L_F/(\sigma + L_F)}.$$ Denoting $$\tilde{\y}_K(\x) = \mathrm{prox}_{s_lg(\x,\cdot)}(\y_K(\x)-s_l\nabla_{\y}f(\x,\y_K(\x))),$$ and $$C_{\y^*(\x)} = \max \left\{ \|\y_0 - \y^*(\x)\|, \frac{s_u}{1-\beta}\|\nabla_\y F(\x,\y^*(\x))\|  \right\},$$ with $\y^*(\x) \in \tilde{\S}(\x)$ and $\x \in \X$. Then it holds that
	\begin{align*}
	\|\y_K(\x) - \y^*(\x)\| &\le C_{\y^*(\x)},\\
	\|\y_K(\x) - \tilde{\y}_K(\x)\| &\le \frac{2C_{\y^*(\x)}(J+2)}{K(1-\beta)},\\
	h(\x,\tilde{\y}_K(\x)) - h^*(\x) &\le \frac{2C_{\y^*(\x)}^2(J+2)}{K(1-\beta)s_l},
	\end{align*}
	where $J = \lfloor 2/(1-\beta) \rfloor$. Further, $\y_K(\x)$ converges to $\tilde{\S}(\x)$ as $K \rightarrow \infty$ for any $\x \in \X$. 
\end{prop}

Step 3: Taking a closer look at the proofs for Proposition~\ref{prop:yK_conver} and Proposition~\ref{prop:varphi_LSC}, we observe that the techniques we used barely rely on the smoothness of the LL objective. Therefore, straightforward extensions of Proposition~\ref{prop:yK_conver} and Proposition~\ref{prop:varphi_LSC} to the nonsmooth case can yield the desired uniform convergence of $\tilde{\y}_K(\x)$ and the UL objective convergence, respectively. 

Step 4: Similar to the arguments in the proof of Theorem~\ref{thm:conver_BDA}, by combining Step 1 and Step 2, we eventually meet the \emph{LL solution set} and \emph{UL objective convergence} properties, and hence the analysis framework in Theorem~\ref{thm:general} has been activated. Therefore, the same convergence results concerning $\{\x_K\}_{K\in\mathbb{N}}$ and $\{\varphi_{K}(\x)\}$ can be achieved as following. 

\begin{thm}
	Suppose Assumption 1 is satisfied, $g$ is continuous and convex w.r.t. $\y$, and let $\{\y_K\}$ be defined as in Eq.~\eqref{eq:lower-non}, $s_l \in (0,1/L_f]$, $s_u \in (0,2/(L_F+\sigma)]$, $$\alpha_k = \min \left\{2\gamma/k(1-\beta),1-\varepsilon \right\},$$ with $ k \ge 1$, $\varepsilon>0$, $\gamma \in (0,1]$ and $$\beta = \sqrt{1-2s_u\sigma L_F/(\sigma + L_F)}.$$ Assume further that $\S(\x)$ is nonempty for any $\x \in \X$ and $\S(\x)$ is continuous on $\X$.
	Then
	\begin{itemize}
		\item[(1)] if $\x_{K}\in\arg\min_{\x\in\X}\varphi_{K}(\x)$, we have the same results as in Theorem \ref{thm:general};
		\item[(2)] if $\x_K$ is a local minimum of $\varphi_{K}(\x)$ with uniform neighborhood modulus $\delta > 0$, we have that any limit point $\bar{\x}$ of the sequence $\{\x_K\}$ is a local minimum of $\varphi$. 
	\end{itemize} 
\end{thm}

\end{document}